\documentclass{article}

\usepackage[authoryear]{natbib}

\usepackage[utf8]{inputenc}  
\usepackage[T1]{fontenc}     
\usepackage{textcomp}        

\usepackage{amssymb,amsmath,amsthm}
\usepackage{multirow}
\usepackage{graphicx}
\usepackage{float}
\usepackage{subcaption}
\usepackage{booktabs}
\newtheorem{theorem}{Theorem}
\usepackage{comment}  
\usepackage{xcolor}
\usepackage{algorithm}
\usepackage{algpseudocode}
\usepackage{hyperref}

\usepackage{graphicx}
\usepackage{subcaption}
\usepackage[margin=1in]{geometry}
\usepackage{lmodern}

\title{Mind the Jumps: A Scalable Robust Local Gaussian Process for Multidimensional  Response Surfaces with Discontinuities}
\author{%
  Isaac Adjetey\thanks{Department of Industrial and Manufacturing Engineering, Florida State University, USA. Email: \texttt{iaa18a@fsu.edu}}%
  \quad and
  Yiyuan She\thanks{Institute for Theoretical Sciences, Westlake University, China. Email:  \texttt{sheyiyuan@westlake.edu.cn}}%
}

\date{} 

\date{} 

\begin{document}

\maketitle

%
%
%
%
%
%
%
%

\begin{abstract}
Modeling response surfaces with abrupt jumps and discontinuities remains a major challenge across scientific and engineering domains. Although Gaussian process models excel at capturing smooth nonlinear relationships, their stationarity assumptions limit their ability to adapt to sudden input-output variations. Existing nonstationary extensions, particularly those based on domain partitioning, often struggle with boundary inconsistencies, sensitivity to outliers, and scalability issues in higher-dimensional settings, leading to reduced predictive accuracy and unreliable parameter estimation.

To address the challenges posed by data heterogeneities and high dimensions, this paper proposes the Robust Local Gaussian Process (RLGP) model, a novel framework that integrates adaptive nearest-neighbor selection with a sparsity-driven robustification mechanism. Unlike existing  methods, RLGP  leverages an optimization-based mean-shift robustification after a multivariate perspective transformation  combined with local neighborhood modeling to mitigate the influence of outliers. This approach enhances predictive accuracy near discontinuities while improving resistance to data heterogeneity.

Comprehensive evaluations on real-world datasets show that RLGP consistently delivers high predictive accuracy and maintains competitive computational efficiency, especially in scenarios with sharp transitions and complex response structures. Scalability tests further confirm RLGP's stability and reliability in higher-dimensional settings, where other methods falter. These outcomes establish RLGP as an effective and practical solution for modeling nonstationary and discontinuous response surfaces, applicable across a wide range of real-world scenarios.
\end{abstract}


\maketitle

\section{Introduction}

Addressing abrupt shifts in response dynamics is a key challenge for \textit{surrogate} (or \textit{emulator}) models, which are trained on carefully selected simulator outputs to provide fast predictions and uncertainty estimates without incurring additional simulation costs \citep{santner2018space}. These models are widely used across engineering and scientific disciplines \citep{gramacy2008bayesian,ba2012composite,heaton2012flexible,dutordoir2017deep}. Capturing the complexity of such systems demands advanced techniques that can represent intricate patterns without oversimplifying the underlying structure. \citep{kennedy2001bayesian, ohagan2006bayesian, santner2018space, kleijnen2018design}.

A major challenge lies in the presence of abrupt shifts in response behavior, which can be triggered by even \emph{subtle} variations in input conditions \citep{oakley2004probabilistic, marrel2009calculations}. For instance,
in aerospace engineering, studies of NASA's Langley Glide-Back Booster (LGBB) have shown that slight changes in re-entry speed, angle of attack, or sideslip angle can cause sharp variations in lift force, highlighting the complexity of modeling aerodynamic behavior \citep{ pamadi2004aerodynamic, gramacy2008bayesian, sauer2022active}. In additive manufacturing, especially metal 3D printing methods like Direct Energy Deposition (DED), slight changes in parameters---such as laser power, scanning speed, or powder feed rate---can significantly affect material properties and part quality \citep{cho2023investigation}. Similarly, in petroleum engineering, where modeling soil permeability is vital for efficient extraction, \cite{kim2005analyzing} reported that minor location changes in the Schneider Buda oil field (Texas) caused abrupt shifts in permeability.

In addition to abrupt shifts, high dimensionality can also give rise to spatially \textit{localized} irregularities \citep{gu2018scaled, pratola2017fast}, where models must scale efficiently to handle dozens---or even hundreds---of input variables to remain practical in real-world applications \citep{liu2020when}.

These challenges are not just theoretical—they form the core motivation for this paper, as they emerge prominently in the real-world case studies explored in this work. In chemical manufacturing, predicting carbon nanotube yield involves modeling abrupt changes near catalyst activation thresholds. In adaptive STEM imaging, reconstructing high-resolution surfaces from partial scans requires robust interpolation across sharp boundaries. In environmental sensing, predicting corrosion current under variable conditions presents localized nonlinearities near physical thresholds. Each of these domains calls for surrogate models that are both robust to discontinuities and scalable in high dimensions—gaps that existing methods struggle to fill.

Standard Gaussian Process (GP) models are highly valued for their flexibility in capturing nonlinear relationships and for providing  statistical predictions and estimates, enabling partially analytic inference \citep{paciorek2006spatial, heinonen2016non}. In particular, the ability to quantify prediction \emph{uncertainty} makes them particularly useful in applications where knowing how reliable a prediction is matters as much as the prediction itself \citep{wang2019nonstationary, neto2020nonstationary}.
However, the \textbf{stationarity} assumption inherent in conventional GP models restricts their ability to adapt to environments with abrupt shifts in input-output dynamics \citep{gramacy2008bayesian}. 
Moreover, inaccurate predictions and unreliable uncertainty estimates jeopardize downstream tasks, including {failure region identification} \citep{wang2016gaussian}, {sequential experimental design} \citep{mckay2000comparison}, {simulator calibration} \citep{kennedy2001bayesian}, and {sensitivity analysis} \citep{rohmer2011global}.

Addressing nonstationary yet  realistic response patterns remains a fundamental challenge in surrogate modeling \citep{heaton2012flexible,pandita2021surrogate}. Efforts to address localized complexities lead to the development of {nonstationary} GP methods, which can be categorized as follows.

\begin{enumerate}
\item \textbf{Kernel-based spatial modeling.}   \cite{higdon1999non}, \cite{paciorek2003nonstationary}, and \cite{katzfuss2013bayesian}  introduce spatially varying covariance structures to model abrupt transitions in the response surface. However, these approaches still display residual correlation for observations near regional boundaries may.

\item
\textbf{Partition models.}  \cite{kim2005analyzing} and \cite{pope2021gaussian} employ Voronoi tessellation to divide the input space into triangular regions, while \cite{luo2021bayesian} uses Delaunay triangulation to construct spatial adjacency graphs. Within each region, an independent stationary GP is fitted, and both the number of partitions and model parameters are inferred jointly using Bayesian sampling. The tessellation-based approaches perform well   when \(d = 2\), but their computational complexity increases significantly with dimensionality, limiting their scalability and applicability to higher-dimensional problems.

Tree-based methods, including   \cite{chipman1998bayesian}, \cite{denison2002bayesian}, \cite{gramacy2008bayesian}, \cite{taddy2011dynamic}, \cite{chipman2013bayesian},  \cite{pratola2014parallel}, \cite{konomi2014adaptive}, and \cite{pope2021gaussian},  divide the input space into axis-aligned regions, within which independent GPs approximate the response surface. These methods perform well in specific low--dimensional settings (e.g., $d\leq 5$). However, as dimensionality increases, their recursive partitioning along coordinate axes leads to combinatorial complexity, making them computationally prohibitive. Additionally, poor partitioning can either excessively \emph{fragment} the space, leaving too little data in each region, or fail to sufficiently divide the space, missing abrupt response jumps. These issues degrade predictive accuracy, particularly near boundaries in higher dimensional settings ($d >10$), where conflicting local trends become more pronounced.
\cite{park2022jump} highlight their limitations in handling real-world variations across complex regional boundaries.

\item \textbf{Neighborhood-based  models.} Emerging methodological advances leverage a ``transductive'' framework \citep{schwaighofer2002transductive}, where the test data itself guides the training process, dynamically refining training data selection to enhance predictive accuracy at individual test locations.
One of the earlier approaches, the local approximate GP \citep{gramacy2015local}, improves computational efficiency by constructing a GP model using only the nearest neighbors of the test point, thereby reducing complexity while maintaining accuracy. More recently, the locally induced GP \citep{cole2021locally} extends this idea by selecting induced points based on local structure rather than relying solely on proximity, offering a more flexible and adaptive representation. Building on these concepts, more recently, the jump GP \citep{park2022jump} further refined local data selection by first identifying relevant neighbors and then segmenting them using a parametric hyperplane or partitioning function. This partitioning function can be linear, quadratic, or even cubic, depending on the complexity of the local data, ensuring that a GP is fitted specifically to the subset containing the test location.
Generally speaking, these neighborhood-based methods   perform well when the test point is deep within a homogeneous region, but  face difficulties  near region boundaries, where conflicting trends from adjacent areas degrade prediction accuracy (cf. Figure~\ref{fig:data_comparison}). This boundary issue becomes more severe in high dimensions, further reducing accuracy and complicating robust inference.
\item \textbf{Probabilistic deep models.} A distinct category of methods learns global, hierarchical representations of the data to implicitly model non-stationarity. Bayesian Neural Networks (BNNs) \citep{jospin2022hands}, for example, place distributions over network weights to capture parameter uncertainty, allowing the model to adapt to complex functions. Deep Gaussian Processes (DeepGPs)  \citep{damianou2013deep} compose multiple GP layers, creating a deep architecture that learns a flexible, non-linear warping of the input space.  While these models are highly expressive and can capture intricate data patterns without explicit partitioning, that power comes at the cost of increased computational complexity, which in turn can affect accuracy as it often requires approximate inference techniques to ensure scalability.

\end{enumerate}

Consequently, a feasible yet unexplored approach is to systematically identify and downweight the influence of extraneous data within a local neighborhood, particularly when dealing with conflicting data patterns  from adjacent regions that exhibit distinct response dynamics.  Indeed, most approaches fail to isolate the most relevant local trends \citep{waelder2024improved}, undermining both interpretability and predictive reliability. Furthermore, their application is often limited to dimensions (e.g., $\le 10$ dimensions) far below application needs. These limitations underscore the need for robust modeling frameworks that explicitly address boundary-induced uncertainty while maintaining scalability in high-dimensional settings.
 \\

To overcome the limitations of existing methods, including  rigid partitioning schemes, sensitivity to boundary-adjacent outliers, and limited scalability in high dimensions, we propose the Robust Local Gaussian Process (\textbf{RLGP}) framework. RLGP introduces  a novel robust Gaussian-process formulation to detect and mitigate the impact of anomalous observations within local neighborhoods. Unlike existing methods, which often struggle with imperfect predefined neighborhoods, RLGP explicitly adjusts response values through an optimization-driven estimation of outlyingness parameters. This novel sparse learning procedure significantly enhances computational efficiency in high-dimensional settings (e.g.,  \( d=500 \)) and achieves  prediction accuracy, particularly near region boundaries where data characteristics vary significantly.

The key contributions of RLGP are outlined below.

\begin{enumerate}
\item
 RLGP introduces a novel robust formulation of  local Gaussian processes, employing a recent \emph{multivariate perspective} transformation and a sparse \textit{mean-shift} parameterization to effectively detect and accommodate various anomalies inconsistent with the response curve model assumptions.

\item
An $\ell_0$-type regularization is utilized to capture the inherent sparsity of the outlier-contaminated Gaussian process, addressing overparameterization issues and boosting computational speed.

\item
RLGP features an optimization-driven algorithm that combines gradient-based block coordinate descent and sparsity-driven iterative quantile thresholding, which guarantees convergence and efficiency.

 \item
Unlike many existing methods that require multiple tuning parameters, our algorithm uses a single, intuitive regularization parameter. A data-adaptive choice of the parameter is provided, which performs well across various scenarios.
\item
RLGP provides superior prediction accuracy and exceptional efficiency for complex response curves across hundreds of dimensions, significantly advancing beyond prior methods typically limited to no more than $  10 $ dimensions and less adept at handling irregularities.

\end{enumerate}

This paper is organized as follows. Section 2 introduces the RLGP framework, presenting its theoretical foundation through adaptive nearest-neighbor subdesigns and a robust GP formulation enhanced by sparse mean-shift parameters to mitigate boundary-driven outliers. Section 3 details RLGPs computational implementation, emphasizing gradient-based optimization techniques that enable scalability in high-dimensional setting.  Section 4 validates the proposed RLGP framework by benchmarking it against state-of-the-art Gaussian Process models. The evaluation covers real-world applications, including carbon nanotube yield, compressed sensing imaging, corrosion sensors, and cancer phenotype analysis.   In addition to these widely used benchmark datasets from the literature, we also conduct simulation studies using synthetic data that exhibit complex patterns with jumps and discontinuities. Furthermore, the experiments are extended to higher-dimensional settings to assess the model's scalability.  Section 5 concludes with a comprehensive summary of the findings.

\textbf{Notations:} The following notations and symbols will be used.
Given a matrix $\boldsymbol{A}$, we use $\|\boldsymbol{A}\|_2$ to denote its spectral norm (the largest singular value of $\boldsymbol{A}$), and $\lambda_{\min}(\boldsymbol{A})$, $\lambda_{\max}(\boldsymbol{A})$ to denote its smallest and largest eigenvalues, respectively. Given a symmetric matrix $\boldsymbol{S}, \boldsymbol{S} \succeq \mathbf{0}$ means that it is positive semi-definite. Given two matrices $\boldsymbol{A}, \boldsymbol{B}$ of the same size, $\boldsymbol{A} . * \boldsymbol{B}$ denotes their elementwise product, and $\langle\boldsymbol{A}, \boldsymbol{B}\rangle$ denotes their inner product. Finally, given $\boldsymbol{A} \succeq \mathbf{0}, \boldsymbol{A}^{1 / 2}$ means its matrix square root.

\section{Robust Local Gaussian Process  Learning}
\subsection{Challenges in Local Gaussian Process Modeling}
The main objective is  to estimate an unknown nonlinear regression function
\( f: X \to \mathbb{R} \), where
\( X \subseteq \mathbb{R}^d \) represents the input domain with dimension $d$. The function
\( f(\boldsymbol{x}) \) is assumed to be piecewise continuous and can be expressed as:
\begin{equation}
f(\boldsymbol{x}) =
\begin{cases}
f_1(\boldsymbol{x}), &\text{if } \boldsymbol{x} \in X_1, \\
f_2(\boldsymbol{x}), &\text{if } \boldsymbol{x} \in X_2, \\
\vdots \\
f_k(\boldsymbol{x}), &\text{if } \boldsymbol{x} \in X_k, \\
\end{cases}
\end{equation}
where \( X_1, X_2, \dots, X_k\)  are disjoint subsets that partition the entire domain \( X \), and  \( f_k(\boldsymbol{x}) \) represents the function for region \( X_k \).  Within  region \( X_k \), we observe \( n_k \) noisy \( (\boldsymbol{x}_i, y_i) \) pairs that can be modeled by:
\begin{equation}
y_i = f_k(\boldsymbol{x}_i) + \epsilon_k(\boldsymbol{x}_i),
\end{equation}
where    \(\epsilon_k(\boldsymbol{x}_1), \ldots, \epsilon_k(\boldsymbol{x}_{n_{k}})\)  are often \textbf{correlated} within each region \( X_k \), although the errors from different regions are often assumed to be independent of each other. The absence of the independent and identically distributed (i.i.d.) errors assumption within a region  complicates the modeling process.
In practice, neither the number of regions  nor \( X_k \)'s boundary is known, while partitioning the domain into proper regions for accurate estimation becomes increasingly challenging even for relatively low dimensions (e.g. \( d \geq 3 \)).

To address these issues, numerous researchers advocate for a higher dimensional analog of  local kernel smoothing  \citep{emery2009kriging,gramacy2015local,park2022jump}. This type of methods selects a set of nearest neighbors around each test location \( \boldsymbol{x}_{\ast} \in X\), forming a focused training subset \( D_{n_{k}}(\boldsymbol{x}_{\ast}) = \{(\boldsymbol{x}_1, y_1), \ldots, (\boldsymbol{x}_{n_{k}}, y_{n_{k}})\} \) to obtain a local fit \( \hat{f}_{k} \) at \(  \boldsymbol{x}_{\ast} \). \textit{If}  it is reasonable to assume  \( D_{n_{k}}(\boldsymbol{x}_{\ast}) \)  is a small subset of  \( X_k \), such that  \( f_{k} \) can be  approximated by a constant, then the local data \( \boldsymbol{y} = [y_1, \ldots, y_{n_{k}}]^T \in \mathbb{R}^{n_k}\)   share a common mean, even though its components are correlated. Consequently, \( \boldsymbol{y} \) can be modeled by a multivariate Gaussian distribution:
\begin{equation} \label{eq:label2}
\boldsymbol{y} \sim \mathcal{N}(\mathbf{1}\,\mu, \boldsymbol{\Sigma}), \;\; \boldsymbol{\Sigma} = \nu \boldsymbol{I} + \boldsymbol{C},
\end{equation}
where \( \boldsymbol{C} = [c( \boldsymbol{x}_i,\boldsymbol{x}_j; \theta) \)] is an \( n_{k} \times n_{k} \) covariance matrix for the data points in \( D_{n_{k}}( \boldsymbol{x}_{\ast}) \). The parameter \( \nu > 0 \) accounts for the variance due to measurement errors, while  \( \boldsymbol{C} \), defined through the covariance function \( c(\cdot, \cdot; \theta) \), captures dependences between the observations. This enables us to borrow information from correlated neighboring points to enhance prediction accuracy. To simplify notation, we omit the subscript \( k \) in \( n_k \), referring to it as \( n \), so \( D_{n_k} \) becomes \( D_n \). The dependence of \( D_n \) on \( \boldsymbol{x}_\ast \) will also be suppressed when clear from context. Additionally, the model changes with the test point, but for clarity, the dependence of all parameters on the test point, \( \boldsymbol{x}_{\ast} \), will be omitted.

The squared exponential (SE) kernel is a popular choice in such local Gaussian processes owing to its smoothness, flexibility, and ability to capture complex patterns in data \citep{paciorek2003nonstationary, duvenaud2014automatic, karimi2020generalized}:
\begin{equation}
\label{eq:label5}
c(\boldsymbol{x}_i, \boldsymbol{x}_j; \theta_0, \vartheta) = \theta_0 \exp \left( -\vartheta ( \boldsymbol{x}_i - \boldsymbol{x}_j )^T (\boldsymbol{x}_i- \boldsymbol{x}_j) \right),
\end{equation}
where \( \boldsymbol{x}_i, \boldsymbol{x}_j \in D_{n} ( \boldsymbol{x}_{\ast})\), \(\theta_0\)  controls the overall variability of the function, and \(\vartheta\) is the so-called concentration  parameter that determines how quickly correlations decay with distance.  It is well known that for  \( c\) defined in \eqref{eq:label5}, the resulting covariance matrix \( \boldsymbol{\Sigma} \) formed is positive semidefinite \citep{williams2006gaussian}.  Although our discussions focus on \eqref{eq:label5}, the proposed methodology in this paper is applicable to  any differentiable covariance function \( c \), such as the Mat\'ern  kernel kernel \citep{stein2012interpolation}.
The parameters \( \mu \) and \( \nu\), and the set of hyperparameters associated with \( c\), can be optimized by maximizing the likelihood function.  The negative log-likelihood  based on the local multivariate Gaussian  model \eqref{eq:label2} is given by
$
\frac{1}{2}(\boldsymbol{y}-\mathbf{1} \,\mu)^T \boldsymbol{\Sigma}^{-1}(\boldsymbol{y}-\mathbf{1} \,\mu)+\frac{n}{2} \log \operatorname{det} \boldsymbol{\Sigma}.
$

Nevertheless, in  real-world scenarios, especially when dimension is higher, the ``local'' data constructed by nearest neighbors often exhibit \textbf{heterogeneity}, where variations  cannot be fully captured by a single Gaussian distribution.
In statistics, heterogeneity can refer to various deviations   from homogeneous modeling assumptions, such as non-constant means or non-constant variances.
In our context, this arises when local neighborhoods $ D_n $ mix data points from adjacent regions with distinct response dynamics. Although data points truly belonging to a single, ideal subregion might theoretically share a constant mean response, constructing such pure local sets is often impractical, particularly near region boundaries or in higher dimensions.
Consequently, the observed local data $ D_n $ typically exhibits non-uniform means,   a key aspect of the heterogeneity we address.  Furthermore, our model is explicitly designed to handle non-constant variances, another prevalent feature of these mixed local datasets. The   heterogeneity in this work  refers  to these combined effects of locally varying means and variances.

The left panel of  Figure \ref{fig:data_comparison} illustrates a homogeneous scenario where \( D_n \) consists of data from the same region, while  the right panel of Figure \ref{fig:data_comparison} depicts a heterogeneous scenario where \( D_n \) contains data from multiple regions.
Intuitively, the  formulation in \eqref{eq:label2} is compatible with homogeneous local data, but  fails miserably in heterogeneous settings.
Indeed, near regional boundaries, \( D_n \) may exhibit multiple modes and abrupt changes, leading to degraded model performance.

To provide the reader with more intuition, let's examine the maximum likelihood estimate (MLE) of \( \mu \) based on the canonical multivariate Gaussian model:
\begin{equation}\label{muhat}
\hat{\mu} = \frac{\mathbf{1}^T \boldsymbol{\Sigma}^{-1} \boldsymbol{y}}{\mathbf{1}^T \boldsymbol{\Sigma}^{-1} \mathbf{1}},
\end{equation}
which represents a \textit{weighted average} of the observations \( y_i \). However, \( \boldsymbol{y} \) often includes \textit{outliers}, such as the yellow and blue points from regions 2 and 3 in the right panel of Figure \ref{fig:data_comparison}, which can  disproportionately influence the estimate  in \eqref{muhat}. In extreme cases, even a single rogue point can distort \( \hat{\mu} \), resulting in a failure to capture the statistical properties of the majority of the data.

\begin{figure}
    \centering
    \begin{subfigure}[b]{0.47\textwidth}
        \centering
        \includegraphics[width=\textwidth]{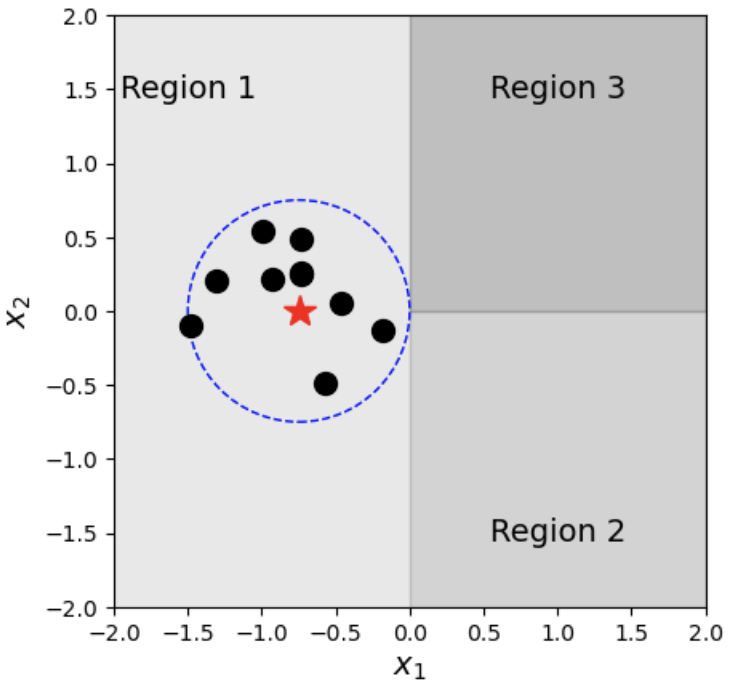}
    \end{subfigure}
    \hfill
    \begin{subfigure}[b]{0.492\textwidth}
        \centering
        \includegraphics[width=\textwidth]{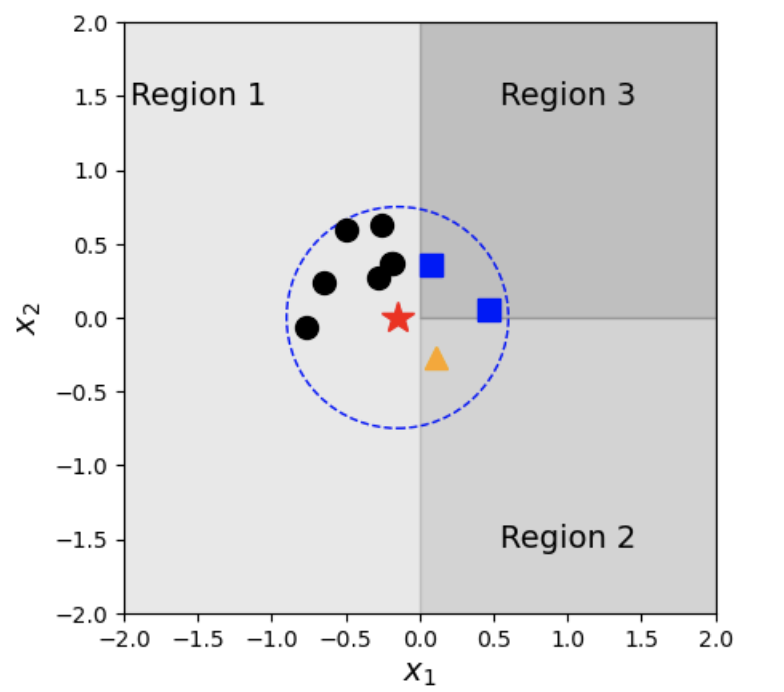}
    \end{subfigure}
    \caption{ Homogeneous local data (left) versus  Heterogeneous local data (right). The circles represent the selected nearest neighbors of the test location (red star). In the homogeneous case, the data points are consistent with a single simple distribution, while the heterogeneous case includes points from multiple regions, exhibiting varying statistical properties and introducing potential anomalies (blue and yellow points).}
    \label{fig:data_comparison}
\end{figure}

To address these challenges,  {mixtures of Gaussians} with distinct means and covariances have been proposed as extensions to model \eqref{eq:label2} \citep{shi2005hierarchical,liu2015kinect,daemi2019gaussian,guan2024mixture}. While effective in some ideal cases, these models often struggle to capture the complexity of data that deviates significantly from Gaussianity, such as heavy tails or skewness. Additionally, fitting mixture models in higher dimensional settings is computationally expensive, and determining the optimal number of components, a critical step to avoid oversimplification or overfitting, poses a significant and non-trivial challenge.

\subsection{Robustification and Perspective Transformation}
Because local training data inevitably become heterogeneous---especially near region boundaries or as the dimensionality increases---we introduce a new, computationally efficient procedure that delivers robust estimates. The method automatically selects the observations most representative of the target region, simultaneously flags severe outliers, and down-weights their influence. This dual strategy makes the model resilient to anomalies, enhances reliability and predictive accuracy, and significantly reduces computational time.

Unlike the rigid mixture models discussed earlier, we adopt a \textbf{sparsity}-oriented   learning framework inspired by  high-dimensional statistics. Within each local neighborhood we introduce an \emph{outlyingness} vector whose elements quantify how atypical each response is., which characterizes the extend of \textit{outlyingness} of each observed response value in a local neighborhood. That is, rather than assuming these outlyingness components are ideally i.i.d. draws from a preset distribution, we estimate them directly as  model parameters, enabling data-driven detection of anomalies.

First, it is easy to derive the negative log-likelihood loss function based on the model in \eqref{eq:label2}  as follows (constants omitted)
\begin{equation}\label{equation8}
\frac{1}{2}(\boldsymbol{y}-\mathbf{1}\, \mu)^T \boldsymbol{\Sigma}^{-1}(\boldsymbol{y}-\mathbf{1} \,\mu)+\frac{1}{2}  \log \det  (\boldsymbol{\Sigma}).
\end{equation}
However, it is well known in robust statistics that  \eqref{equation8} is a poor starting point for robustification: it has no finite lower bound and can diverge as
$\boldsymbol{\Sigma}\to\mathbf{0}$, complicating concomitant covariance (or scale) estimation.  In the isotropic (univariate) case $\boldsymbol {\Sigma} = \sigma^2 \boldsymbol{I}$,
\citet{huber1981robust} elegantly addressed this challenge by replacing the negative log-likelihood   $ \|\boldsymbol{y}-\mathbf{1}\, \mu\|_2^2/\sigma^2+n  \log \sigma^2$    with    $ \|\boldsymbol{y}-\mathbf{1}\, \mu\|_2^2 /\sigma+n\sigma$ (up to multiplicative constants). The    reformulation is precisely the \textit{perspective} transformation \citep{boyd2004convex} of the function $f(\boldsymbol z)=\| \boldsymbol  z\|_2^2 + n$, defined as $g(\boldsymbol z, \sigma) = f(\boldsymbol z/\sigma)\sigma$,  evaluated at the residual vector $ \boldsymbol{y}-\mathbf{1}\, \mu $.    The resulting objective remains bounded as $\sigma \to 0$ and is jointly convex in $(\mu, \sigma)$, providing a stable basis for concomitant location/scale estimation and subsequent robustification.

Seeking an analogous stabilization for the multivariate case, we need a transformation that similarly prevents the objective from diverging as $\boldsymbol{\Sigma}$ approaches singularity and ideally retains desirable properties like convexity. Motivated by Huber's approach, we adopt a \textit{multivariate perspective transformation}, recently proposed by \cite{ebadian2011perspectives} and \cite{effros2014non}, for handling the covariance matrix  $ \boldsymbol{\Sigma}$.
Specifically, we   introduce the positive-definite \( \boldsymbol{S} = \boldsymbol{\Sigma}^{1 / 2} \succ \mathbf{0} \)  and  consider the jointly convex, lower-bounded surrogate:
\begin{equation}
\label{equation88}
 \frac{1}{2}(\boldsymbol{y}-\mathbf{1} \,\mu)^T \boldsymbol{S}^{-1}(\boldsymbol{y}-\mathbf{1} \,\mu)+\frac{1}{2}  \operatorname{Tr}(\boldsymbol{S}).
\end{equation}
To confirm the validity of this parameterization, at the population level, with
  $\boldsymbol{r} \sim \mathcal N(\boldsymbol{0}, \boldsymbol{\Sigma})$, the objective is $
 \mathbb{E}  \left [ \frac{1}{2} \boldsymbol{r}^T \boldsymbol{S}^{-1} \boldsymbol{r} +\frac{1}{2}  \operatorname{Tr}(\boldsymbol{S})\right]
$.
    Differentiating with respect to $\boldsymbol{S}$ yields $   - \frac{1}{2}  \boldsymbol{S}^{-1}  \mathbb{E}  [\boldsymbol{r} \boldsymbol{r}^T] \boldsymbol{S}^{-1} + \frac{1}{2}\boldsymbol{I} = 0$. By strict convexity, the unique minimizer is the  positive-definite root   $  \boldsymbol{\Sigma}^{1 / 2}$, which justifies the parameterization.


Next, we robustify \eqref{equation88}. In the presence of gross outliers, our goal is to \emph{neutralize} their influence on the criterion irrespective of the observed values.
Write   $\boldsymbol{r} = \boldsymbol{y}-\mathbf{1}\, \mu$ and set  $\boldsymbol{S} = [\tau_{i,j}]$ (symmetric). Then   $$\frac{1}{2}(\boldsymbol{y}-\mathbf{1} \,\mu)^T \boldsymbol{S}^{-1}(\boldsymbol{y}-\mathbf{1} \,\mu) =\frac{1}{2} \sum_{i,j} r_i r_j \tau_{i,j},$$ The portion involving, say, observation $1$ is
$$    \frac{\tau_{1,1}}{2 } r_1^2 + ( \sum_{j\ne 1}r_j \tau_{1,j} ) r_1,$$ where $\tau_{1,1}>0$ since $\boldsymbol{S} \succ  \boldsymbol{0}$. To make this attain its minimum \textit{regardless of}   the values of  $y_1$ and $\mu$, we introduce an observation-specific adjustment and redefine   $r_1 '= y_1 - \mu - \gamma_1$. Then, choosing $\gamma_1$   so that $
  \partial /\partial \gamma_1 ( (\tau_{1,1}/2 ) r_1'^2 + ( \sum_{j\ne 1}r_j \tau_{1,j} ) r_1')=0$ can completely remove the first observation's contribution.  Applying this construction to all $i$  yields an adjustment vector    $\boldsymbol{\gamma}\in \mathbb R^n$. Furthermore, since contamination is atypical, we enforce  sparsity in $\boldsymbol {\gamma}$ using      $\ell_0$ regularization    \citep{she2022gaining}, which  leads to      the following optimization problem for estimating all the unknown model parameters:
\begin{equation}
\begin{split}
\label{equation888}
\min _{(\mu, \boldsymbol{\gamma}, \nu, \theta_0,\vartheta)} & \frac{1}{2}(\boldsymbol{y}-\mathbf{1} \,\mu-\boldsymbol{\gamma})^T \boldsymbol{S}^{-1}(\boldsymbol{y}-\mathbf{1} \,\mu-\boldsymbol{\gamma})+\frac{c_{0}}{2}  \operatorname{Tr}(\boldsymbol{S})\\
\text { s.t. } & \|\boldsymbol{\gamma}\|_0 \leq q, \boldsymbol{S}=\boldsymbol{\Sigma}^{1 / 2},\,\boldsymbol{\Sigma}=\nu \boldsymbol{I}+\boldsymbol{C},\, \nu > 0,
\end{split}
\end{equation}
where \( \boldsymbol{C} \) is an \( n \times n \) matrix containing the values of the covariance function evaluated for all pairs of the local data points, and \( q \) (with \( q \leq n/2 \)) controls the sparsity level in the vector \( \boldsymbol{\gamma} \). When \( \gamma_i = 0 \), the \( i \)-th observation is treated as ``clean'' and included in parameter estimation without adjustment. In contrast, when \( \gamma_i \neq 0 \), the corresponding observation \( y_i \) is identified as an outlier. Although the inclusion of \( \boldsymbol{\gamma} \in \mathbb{R}^n \) may seem to overparameterize the model, this is mitigated by the assumption that most observations are not outliers, meaning most entries of \( \boldsymbol{\gamma} \) are expected to be zero. The sparsity constraint on \( \boldsymbol{\gamma} \) ensures that the estimation problem remains well-posed, yielding meaningful parameter estimates and enabling effective outlier detection.
To allow for a possible sample-size correction after anomaly removal, we include the constant $c_0$ in \eqref{equation888} which is typically set to $1$  or $(n-q)/n$.

Notably, the $ \ell_0$ constraint does not limit the magnitude of  \( \gamma_i \), making it effective in minimizing the influence of \( y_i \)  when it conflicts with the test point. In practice, \( q \), which serves as an upper bound of the number of anomalies, is straightforward to specify and is not a sensitive parameter.

In contrast to existing methods such as the Jump Gaussian Process (JGP)  \citep{park2022jump}, which struggles with dimensions \(d \geq 10\), our proposed optimization criterion in \eqref{equation888} enables efficient and scalable algorithms. For instance, the processing time remains approximately half a second per test point, even for hundreds of dimensions.    A detailed comparison of the efficiency and accuracy of our method relative to existing approaches is presented in Section~\ref{sec:methodology}.

\section{Optimization Algorithm } \label{sec:optimization}
Recall the optimization problem in \eqref{equation888}:
\begin{equation*}
\begin{split}
\min _{(\mu, \boldsymbol{\gamma}, \nu, \theta_0,\vartheta)} & \frac{1}{2}(\boldsymbol{y}-\mathbf{1} \,\mu-\boldsymbol{\gamma})^T \boldsymbol{S}^{-1}(\boldsymbol{y}-\mathbf{1} \,\mu-\boldsymbol{\gamma})+\frac{c_{0}}{2}  \operatorname{Tr}(\boldsymbol{S}) \equiv l(\mu, \boldsymbol{\gamma}, \boldsymbol{S})  \\
\text { s.t. } & \|\boldsymbol{\gamma}\|_0 \leq q, \boldsymbol{S}=\boldsymbol{\Sigma}^{1 / 2},\,\boldsymbol{\Sigma}=\nu \boldsymbol{I}+\boldsymbol{C},\, \nu > 0.
 \end{split}
\end{equation*}
With a slight notation abuse, the loss  $l(\mu, \boldsymbol{\gamma}, \boldsymbol{S})$  is also denoted by $l(\mu, \boldsymbol{\gamma}, \nu, \theta_0,\vartheta)$.
Recall the covariance function definition in \eqref{eq:label5}.  Define   $d_{i, j}=(\boldsymbol x_i - \boldsymbol x_j)^T (\boldsymbol  x_i - \boldsymbol  x_j)$  and
\begin{equation}\label{covfunc}
\begin{split}
 \boldsymbol{C}=\theta_0 \boldsymbol{E}=\theta_0 \exp (-\vartheta \boldsymbol{D}), \;\; \boldsymbol{D}=\left[d_{i, j}\right],
\end{split}
\end{equation}
where $\exp (-\vartheta \boldsymbol{D})$ is applied componentwise and $\boldsymbol x_i$ are  located within  $D_n$ given a test  location $\boldsymbol x_*$. The following discussion concentrates on this widely used form of covariance structure in \eqref{covfunc} but the optimization algorithm  can be applied to  any  differentiable  $c(\cdot, \cdot)$.   Notably,  only the formation of  $\boldsymbol D$ (which  can be precomputed) depends  on the dimensions of $\boldsymbol{x}_i$.

To facilitate algorithm design, we introduce some matrix functions
\begin{equation}\label{eq:label15}
\begin{split}
\boldsymbol{E}(\vartheta)=\exp (-\vartheta \boldsymbol{D}),\;\; \boldsymbol{\Sigma}\left(\nu, \theta_0, \vartheta\right)=\nu \boldsymbol{I}+\theta_0 \boldsymbol{E}(\vartheta),\;\; \boldsymbol{S}\left(\nu, \theta_0, \vartheta\right)=\left\{\boldsymbol{\Sigma}\left(\nu, \theta_0, \vartheta\right)\right\}^{1 / 2}.
\end{split}
\end{equation}
From this point forward, for clarity, the dependence of \( \boldsymbol{S} \) on \( (\nu, \theta_0, \vartheta) \) and of \( \boldsymbol{E} \) on \( \vartheta \) will be omitted.  The gradients  of $l$ can be computed involving the matrix powers of $\boldsymbol{S}$ as follows.

\begin{theorem} \label{results1}
Let  $l(\mu, \boldsymbol{\gamma},\nu,\theta_0,\vartheta)$ be the loss function as defined in \eqref{equation88}. The gradient of  $l$  with respect to  $\mu, \boldsymbol{\gamma}, \nu, \theta_0$ and $\vartheta$ are given by:
\begin{equation*}
\begin{split}
\nabla_\mu l \left(\mu, \boldsymbol{\gamma}, \nu, \theta_0, \vartheta\right)& =\left\langle\mathbf{1}, \boldsymbol{S}^{-1}(\mathbf{1} \,\mu-\boldsymbol{y}+\boldsymbol{\gamma})\right\rangle\\
\nabla_{\boldsymbol{\gamma} } l \left(\mu, \boldsymbol{\gamma}, \nu, \theta_0, \vartheta\right)&=\boldsymbol{S}^{-1}(\boldsymbol{\gamma}-(\boldsymbol{y}-\mathbf{1}\, \mu))\\
\nabla_\nu l \left(\mu, \boldsymbol{\gamma}, \nu, \theta_0, \vartheta\right) & =\frac{1}{4}(\boldsymbol{y}-\mathbf{1}\, \mu-\boldsymbol{\gamma})^T \boldsymbol{S}^{-3}(\boldsymbol{y}-\mathbf{1} \,\mu-\boldsymbol{\gamma})+\frac{c_{0}}{4} \operatorname{Tr}\left(\boldsymbol{S}^{-1}\right) \\
\nabla_{\theta_0} l\left(\mu, \boldsymbol{\gamma}, \nu, \theta_0, \vartheta\right) & =\frac{1}{4}\left\langle-\boldsymbol{S}^{-3 / 2}(\boldsymbol{y}-\mathbf{1}\, \mu-\boldsymbol{\gamma})(\boldsymbol{y}-\mathbf{1} \,\mu-\boldsymbol{\gamma})^T \boldsymbol{S}^{-3 / 2}+c_{0}\boldsymbol{S}^{-1}, \boldsymbol{E}\right\rangle \\
\nabla_{\vartheta} l \left(\mu, \boldsymbol{\gamma}, \nu, \theta_0, \vartheta\right) & =\frac{1}{4}\left\langle-\boldsymbol{S}^{-3 / 2}(\boldsymbol{y}-\mathbf{1}\, \mu-\boldsymbol{\gamma})(\boldsymbol{y}-\mathbf{1}\, \mu-\boldsymbol{\gamma})^T \boldsymbol{S}^{-3 / 2}+c_{0} \boldsymbol{S}^{-1},-\theta_0 \boldsymbol{D} . * \boldsymbol{E}\right\rangle.
\end{split}
\end{equation*}
\end{theorem}

\begin{proof}[Proof of Theorem \ref{results1}]
Recall
\begin{equation*}
l(\mu, \boldsymbol{\gamma}, \boldsymbol{S}) = \frac{1}{2}(\boldsymbol{y}-\mathbf{1} \,\mu-\boldsymbol{\gamma})^T \boldsymbol{S}^{-1}(\boldsymbol{y}-\mathbf{1} \,\mu-\boldsymbol{\gamma})+\frac{c_{0}}{2}  \operatorname{Tr}(\boldsymbol{S}).
\end{equation*}
The gradient of $l$ with respect to \( \mu \) is given by
\begin{equation}
\nabla_\mu l= - \mathbf{1}^T \boldsymbol{S}^{-1} (\boldsymbol{y} - \mathbf{1}\, \mu - \boldsymbol{\gamma}) =   \left\langle \mathbf{1}, \boldsymbol{S}^{-1} (\mathbf{1} \,\mu - \boldsymbol{y} + \boldsymbol{\gamma}) \right\rangle.
\end{equation}
For the gradient of $l$ with respect to  \( \boldsymbol{\gamma} \), we get:
\begin{equation}
\nabla_{\gamma} l  = - \boldsymbol{S}^{-1} (\boldsymbol{y} - \mathbf{1}\, \mu - \boldsymbol{\gamma}) = \left\langle \boldsymbol{S}^{-1}, \boldsymbol{\gamma} - (\boldsymbol{y} - \mathbf{1} \,\mu) \right\rangle.
\end{equation}
The gradient of \( l \) with respect to \( \boldsymbol{S} \) is:
\begin{equation}
 \nabla_{\boldsymbol{S}}l= -\frac{1}{2} \boldsymbol{S}^{-1} (\boldsymbol{y} - \mathbf{1}\, \mu - \boldsymbol{\gamma}) (\boldsymbol{y} - \mathbf{1}\, \mu - \boldsymbol{\gamma})^T \boldsymbol{S}^{-1} + \frac{c_{0}}{2} \boldsymbol{I}.
\end{equation}
Since  \( \boldsymbol{S} = \boldsymbol{\Sigma}^{1 / 2} \),
\begin{equation}
d\boldsymbol{S} = \frac{1}{2} \boldsymbol{\Sigma}^{-\frac{1}{4}}d\boldsymbol{\Sigma} \boldsymbol{\Sigma}^{-\frac{1}{4}}=\frac{1}{2} \boldsymbol{S}^{-\frac{1}{2}}d\boldsymbol{\Sigma} \boldsymbol{S}^{-\frac{1}{2}}
\end{equation}

Treating  \( \boldsymbol{\Sigma} = \nu \boldsymbol{I} + \theta_0 \boldsymbol{E}(\vartheta) \) (cf. \eqref{eq:label15}) as a function of $\nu$, we get  \( d \boldsymbol{\Sigma} =  \boldsymbol{I}\cdot d\nu  \).
Now, applying the chain rule,
\begin{equation}
\begin{split}
\nabla_{\boldsymbol{\nu}} l &= \frac{1}{2} \operatorname{Tr} \left( \left( -\frac{1}{2} \boldsymbol{S}^{-1} (\boldsymbol{y} - \mathbf{1}\, \mu - \boldsymbol{\gamma}) (\boldsymbol{y} - \mathbf{1} \,\mu - \boldsymbol{\gamma})^T \boldsymbol{S}^{-1} + \frac{c_{0}}{2} \boldsymbol{I} \right) \boldsymbol{S}^{-1} \right).\\
&=\frac{1}{4}(\boldsymbol{y}-\mathbf{1}\, \mu-\boldsymbol{\gamma})^T \boldsymbol{S}^{-3}(\boldsymbol{y}-\mathbf{1} \,\mu-\boldsymbol{\gamma})+\frac{c_{0}}{4} \operatorname{Tr}\left(\boldsymbol{S}^{-1}\right).
\end{split}
\end{equation}

Similarly, for the gradient of $l$ with respect to \( \theta_0 \), we obtain
\begin{equation}
\begin{split}
\nabla_{\theta_0} l &= \frac{1}{2} \operatorname{Tr} \left( \left( -\frac{1}{2} \boldsymbol{S}^{-1} (\boldsymbol{y} - \mathbf{1} \,\mu - \boldsymbol{\gamma}) (\boldsymbol{y} - \mathbf{1} \,\mu - \boldsymbol{\gamma})^T \boldsymbol{S}^{-1} + \frac{c_{0}}{2} \boldsymbol{I}\right) \boldsymbol{S}^{-1} \cdot \boldsymbol{E}(\vartheta) \right)\\
&= \frac{1}{4} \left\langle -\boldsymbol{S}^{-3/2} (\boldsymbol{y} - \mathbf{1} \,\mu - \boldsymbol{\gamma}) (\boldsymbol{y} - \mathbf{1}\, \mu - \boldsymbol{\gamma})^T \boldsymbol{S}^{-3/2} + c_{0}\boldsymbol{S}^{-1}, \boldsymbol{E}(\vartheta) \right\rangle.
\end{split}
\end{equation}

Finally,  we calculate  the gradient of $l$ with respect to \( \vartheta \).  From \( \boldsymbol{E}(\vartheta) = \exp(-\vartheta \boldsymbol{D}) \), we know
\begin{equation}\label{eq:label20}
\nabla_{\vartheta} \boldsymbol{E}= - \boldsymbol{D} .* \boldsymbol{E} \; \implies \nabla_{\vartheta} \boldsymbol{\Sigma}=-\theta_0 \boldsymbol{D} .* \boldsymbol{E}.
\end{equation}
Applying the chain rule,  we have
\begin{equation}
\begin{split}
\nabla_{\vartheta} l &= \frac{1}{2} \operatorname{Tr} \left( \left( -\frac{1}{2} \boldsymbol{S}^{-1} (\boldsymbol{y} - \mathbf{1} \,\mu - \boldsymbol{\gamma}) (\boldsymbol{y} - \mathbf{1}\, \mu - \boldsymbol{\gamma})^T \boldsymbol{S}^{-1} + \frac{c_{0}}{2} \boldsymbol{I} \right) \boldsymbol{S}^{-1} \cdot (-\theta_0 \boldsymbol{D} .* \boldsymbol{E}) \right)\\
 &= \frac{1}{4} \left\langle -\boldsymbol{S}^{-3/2} (\boldsymbol{y} - \mathbf{1} \,\mu - \boldsymbol{\gamma}) (\boldsymbol{y} - \mathbf{1}\, \mu - \boldsymbol{\gamma})^T \boldsymbol{S}^{-3/2} +c_{0}  \boldsymbol{S}^{-1}, -\theta_0 \boldsymbol{D} .* \boldsymbol{E} \right\rangle.
\end{split}
\end{equation}
The proof is complete.
\end{proof}

To develop a scalable optimization algorithm suitable for multidimensional applications, we employ block coordinate descent (BCD) to iteratively optimize the parameters. Specifically, after dividing the parameters into three blocks, \( \mu \), \( \boldsymbol{\gamma} \), and \( \boldsymbol{\chi} = \{\nu, \theta_0, \vartheta\} \), our algorithm proceeds as follows:
\begin{equation}
\begin{split} \label{eq:label12b}
\mu^{(t+1)} &= \arg \min_{\mu} l(\mu, \boldsymbol{\gamma}^{(t)},\nu^{(t)}, \theta_0^{(t)}, \vartheta^{(t)})\\
\boldsymbol{\gamma}^{(t+1)} &= \arg \min_{\boldsymbol{\gamma}} l(\mu^{(t+1)}, \boldsymbol{\gamma}, \nu^{(t)}, \theta_0^{(t)}, \vartheta^{(t)})\;\;
\text{s.t.} \;\; \|\boldsymbol{\gamma}\|_0 \leq q \\
\boldsymbol{\chi}^{(t+1)}&= \arg \min_{\boldsymbol{\chi}} l(\mu^{(t+1)},\boldsymbol{\gamma}^{(t+1)},\nu, \theta_0, \vartheta).
\end{split}
\end{equation}
Based on previous discussions, with  $\boldsymbol{\gamma}$ and $\boldsymbol{\chi}$ held fixed,   the solution  for \( \mu \) is
\begin{equation}\label{equationlg1}
\mu^{(t+1)}=\frac{\mathbf{1}^T (\boldsymbol{S}^{(t)})^{-1}(\boldsymbol{y}-\boldsymbol{\gamma}^{(t)})}{\mathbf{1}^T (\boldsymbol{S}^{(t)})^{-1} \mathbf{1}}.
\end{equation}
Fixing \( \mu \) and \( \boldsymbol{\gamma} \), the optimization problem becomes smooth. Therefore, with the gradient formulas provided in Theorem \ref{results1}, one can efficiently optimize all parameters in   $\boldsymbol{\chi}$  using gradient descent or more preferably, quasi-Newton methods. Below, we focus on the optimization of  the  \( \boldsymbol{\gamma} \)-block.

\paragraph*{\( \boldsymbol{\gamma} \)-optimization}
The sub-optimization problem  can be formulated as
\begin{equation}\label{eq:label12}
\min_{\boldsymbol{\gamma}}
\frac{1}{2}(\boldsymbol{y} - \mathbf{1} \,\mu - \boldsymbol{\gamma})^{T} \boldsymbol{S}^{-1}(\boldsymbol{y} - \mathbf{1} \,\mu - \boldsymbol{\gamma})
=l(\boldsymbol{\gamma})\;\; \text{s.t.} \;\; \|\boldsymbol{\gamma}\|_0 \leq q.
\end{equation}
The presence of the  discrete, nonconvex $\ell_0$ constraint complicates the direct minimization of \eqref{eq:label12}.  To tackle this, we  first construct a \textit{surrogate function} $g$,   to facilitate the  optimization of  \( \boldsymbol{\gamma}\).
\begin{theorem}\label{lma1}
Let \( g(\boldsymbol{\gamma}, \boldsymbol{\gamma}^{-}) = l(\boldsymbol{\gamma}^-) +\langle \nabla l(\boldsymbol{\gamma}^-), \boldsymbol \gamma-\boldsymbol{\gamma}^{-}\rangle  + \rho \|\boldsymbol{\gamma} -  \boldsymbol{\gamma}^{-}\|_2^2/2 \). Assume  $ \rho \geq \|\boldsymbol{S}^{-1}\|_2 = 1/ {\lambda_{\min}(\boldsymbol{S})}$. Then   \( g(\boldsymbol{\gamma}, \boldsymbol{\gamma}^{-}) \) satisfies $g(\boldsymbol{\gamma}, \boldsymbol{\gamma}^{-}) \geq l(\boldsymbol{\gamma}) $ and $ g(\boldsymbol{\gamma}, \boldsymbol{\gamma}) = l(\boldsymbol{\gamma})$ for all  $\boldsymbol{\gamma}, \boldsymbol{\gamma}^{-}$.
\end{theorem}

\begin{proof}[Proof of Theorem \ref{lma1}]
The gradient of \( g(\boldsymbol{\gamma}, \boldsymbol{\gamma}^-) \) with respect to \( \boldsymbol{\gamma} \) is
\begin{equation}
\nabla g(\boldsymbol{\gamma}, \boldsymbol{\gamma}^-) = \nabla l(\boldsymbol{\gamma}^-) + \rho (\boldsymbol{\gamma} - \boldsymbol{\gamma}^-).
\end{equation}
The Hessian of \( g(\boldsymbol{\gamma}, \boldsymbol{\gamma}^-) \) with respect to \( \boldsymbol{\gamma} \) is thus
\begin{equation}
\nabla^2 g(\boldsymbol{\gamma}, \boldsymbol{\gamma}^-) = \rho \mathbf{I}.
\end{equation}
It is easy to see that  \( l(\boldsymbol{\gamma}) \) is twice differentiable, and its Hessian is given by:
\begin{equation}
\nabla^2 l(\boldsymbol{\gamma}) = \boldsymbol{S}^{-1}.
\end{equation}
From the assumption \( \rho \geq \|\boldsymbol{S}^{-1}\|_2 = 1 / \lambda_{\min}(\boldsymbol{S}) \), we have
$
\nabla^2 g(\boldsymbol{\gamma}, \boldsymbol{\gamma}^-)  \succeq \nabla^2 l(\boldsymbol{\gamma})
$.\\
Therefore,
\begin{equation}
\frac{1}{2} (\boldsymbol{\gamma} - \boldsymbol{\gamma}^-)^T \nabla^2 l(\boldsymbol{\gamma}^-) (\boldsymbol{\gamma} - \boldsymbol{\gamma}^-) \leq \frac{\rho}{2} \|\boldsymbol{\gamma} - \boldsymbol{\gamma}^-\|_2^2.
\end{equation}
Because $l$ is quadratic in $\gamma$,
\begin{equation}
l(\boldsymbol{\gamma}) = l(\boldsymbol{\gamma}^-) + \langle \nabla l(\boldsymbol{\gamma}^-), \boldsymbol{\gamma} - \boldsymbol{\gamma}^- \rangle + \frac{1}{2} (\boldsymbol{\gamma} - \boldsymbol{\gamma}^-)^T \nabla^2 l(\boldsymbol{\gamma}^-) (\boldsymbol{\gamma} - \boldsymbol{\gamma}^-).
\end{equation}
Correspondingly,
\begin{equation}
g(\boldsymbol{\gamma}, \boldsymbol{\gamma}^-) \geq l(\boldsymbol{\gamma}).
\end{equation}

Furthermore,  substituting \( \boldsymbol{\gamma}^- = \boldsymbol{\gamma} \) into the definition of \( g \) gives
$g(\boldsymbol{\gamma}, \boldsymbol{\gamma}) = l(\boldsymbol{\gamma})$. Hence, \( g(\boldsymbol{\gamma}, \boldsymbol{\gamma}^-) \) satisfies both \( g(\boldsymbol{\gamma}, \boldsymbol{\gamma}^-) \geq l(\boldsymbol{\gamma}) \) and \( g(\boldsymbol{\gamma}, \boldsymbol{\gamma}) = l(\boldsymbol{\gamma}) \), as required.
\end{proof}

Under this choice of \( \rho \),  the two properties satisfied by \( g \)  in Theorem~\ref{lma1}  ensures that \( g \) serves as a valid surrogate function for \( l(\boldsymbol{\gamma}) \), enabling the following iterative algorithm for solving \eqref{eq:label12}:
\begin{equation}
\boldsymbol{\gamma}^{(t+1)} = \arg \min_{\boldsymbol{\gamma}}\, g(\boldsymbol{\gamma}, \boldsymbol{\gamma}^{(t)})\
\text{s.t.} \ \|\boldsymbol{\gamma}\|_0 \leq q. \label{eq:label11}
\end{equation}

To address \eqref{eq:label11},  we can perform  \textit{quantile-thresholding} \citep{she2023slow}  iteratively. Specifically,  for any vector \( \boldsymbol{s} = [s_1, \dots, s_p]^{\top} \in \mathbb{R}^p \),  \( \Theta^{\#}(\boldsymbol{s}; q) = [t_1, \dots, t_p]^{\top} \), where $
t_{(j)} =
{s_{(j)}},  \text{if } 1 \leq j \leq q$ and 0 otherwise,
with \( s_{(1)}, \dots, s_{(p)} \) being the order statistics of \( s_1, \dots, s_p \), satisfying \( s_{(1)} \geq s_{(2)} \geq \cdots \geq s_{(p)} \), and \( t_{(1)}, \dots, t_{(p)} \) defined similarly. We may express $g$ as follows:
\begin{equation}
\begin{split}
g(\boldsymbol{\gamma},\boldsymbol{\gamma}^{(t)})
&=l(\boldsymbol{\gamma}^{(t)}) + \frac{\rho }{2} (  || \boldsymbol{\gamma}-\boldsymbol{\gamma}^{(t)}||_{2}^{2}+  \frac{2 }{\rho}\langle
 \nabla_{\boldsymbol{\gamma}} l(\boldsymbol{\gamma}^{(t)}) , \boldsymbol{\gamma}-\boldsymbol{\gamma}^{(t)} \rangle  )\\
&=l(\boldsymbol{\gamma}^{(t)}) + \frac{\rho }{2} ( || (\boldsymbol{\gamma}-\boldsymbol{\gamma}^{(t)}) +\frac{1}{\rho}\nabla_{\boldsymbol{\gamma}} l(\boldsymbol{\gamma}^{(t)})||_{2}^{2}   - (
   \frac{1}{\rho}\nabla_{\boldsymbol{\gamma}} l(\boldsymbol{\gamma}^{(t)}))^{2} )\\
   &= \frac{\rho }{2}  || (\boldsymbol{\gamma}-(\boldsymbol{\gamma}^{(t)} -\frac{1}{\rho}\nabla_{\boldsymbol{\gamma}} l(\boldsymbol{\gamma}^{(t)}))||_{2}^{2}  +  l(\boldsymbol{\gamma}^{(t)}) -
   \frac{1}{2\rho}(\nabla_{\boldsymbol{\gamma}} l(\boldsymbol{\gamma}^{(t)}))^{2}  \\
   &= \frac{\rho }{2}  || (\boldsymbol{\gamma}-(\boldsymbol{\gamma}^{(t)} -\frac{1}{\rho}\nabla_{\boldsymbol{\gamma}} l(\boldsymbol{\gamma}^{(t)}))||_{2}^{2}  +h(\boldsymbol{\gamma}^{(t)}).
\end{split}
\end{equation}
Since $h$ is independent of $\boldsymbol{\gamma}$, we can redefine $\boldsymbol{\gamma}^{(t+1)} $ in  \eqref{eq:label11}  in an equivalent form:
\begin{equation} \label{equation11g}
\begin{split}
\boldsymbol{\gamma}^{(t+1)} = \operatorname*{argmin}_{\boldsymbol{\gamma}}  \frac{1 }{2}  || \boldsymbol{\gamma}-(\boldsymbol{\gamma}^{(t)}- \frac{1}{\rho}\nabla_{\boldsymbol{\gamma}} l(\boldsymbol{\gamma}^{(t)}))||_{2}^{2} \;\;\text { s.t. }\;\; \|\boldsymbol{\gamma}\|_0 \leq q.
\end{split}
\end{equation}
Now a globally optimal solution to \eqref{eq:label11} can be effectively achieved through:
\begin{equation} \label{equationlg2}
\boldsymbol{\gamma}^{(t+1)} \leftarrow \Theta^{\#}(\boldsymbol{\xi}; q),\;\;\boldsymbol{\xi} = \boldsymbol{\gamma}^{(t)} - \frac{1}{\rho} \nabla_{\boldsymbol{\gamma}} l(\boldsymbol{\gamma}^{(t)}),
\end{equation}
where $\rho$ should be at least $ 1/ \lambda_{\min}(\boldsymbol{S}^{(t)})$ according to Theorem \ref{lma1}.

Based on the algorithm design, the following function value decreasing property is maintained.
\begin{theorem}\label{thrm1}
Suppose \(\mu^{(0)}, \boldsymbol{\gamma}^{(0)}\)  and  \(\boldsymbol{S}^{(0)}\)  are feasible. Then for the sequence of iterates defined in  \eqref{eq:label12b}, the following property holds:
\begin{equation}
l(\mu^{(t+1)}, \boldsymbol{\gamma}^{(t+1)}, \boldsymbol{S}^{(t+1)}) \leq l(\mu^{(t)}, \boldsymbol{\gamma}^{(t)}, \boldsymbol{S}^{(t)})
\end{equation}
and $\|\boldsymbol{\gamma}^{(t+1)}\|_0\le q $  for all \( t \geq 0 \). Therefore, the  value of the objective function decreases monotonically, ensuring convergence.
\end{theorem}

The step-by-step implementation of the algorithm is summarized in Algorithm \ref{alg:MYALG}.
Based on \eqref{eq:label2},
the posterior  distribution  of the response \( f \) at a test location \( \boldsymbol{x}_* \)  is given by  (details ommitted):
\begin{equation} \label{eq:conditional_distribution}
\begin{split}
f(\boldsymbol{x}_*) | \mathcal{D}_{n} &\sim \mathcal{N}( \mu(\boldsymbol{x}_* ), \sigma^2(\boldsymbol{x}_* ) ) \;\; \text{with}  \\
 \mu(\boldsymbol{x}_* ) &= \mu + \boldsymbol{c}_{\ast}^T \boldsymbol{\Sigma}^{-1} (\boldsymbol{y} - \mathbf{1}\, \mu - \boldsymbol{\gamma}), \\
\sigma^2(\boldsymbol{x}_* ) &= c(\boldsymbol{x}_*, \boldsymbol{x}_* ; \theta_0, \vartheta) - \boldsymbol{c}_{\ast}^T \boldsymbol{\Sigma}^{-1} \boldsymbol{c}_{\ast},
\end{split}
\end{equation}
where  \(\boldsymbol{c}_{\ast} = [c(\boldsymbol{x}_1, \boldsymbol{x}_* ; \theta_0, \vartheta), c(\boldsymbol{x}_2, \boldsymbol{x}_* ; \theta_0, \vartheta), \ldots, c(\boldsymbol{x}_n, \boldsymbol{x}_* ; \theta_0, \vartheta)]^T\in \mathbb{R}^{n}\), with the components representing the  covariance  function values between the local data  and $\boldsymbol{x}_*$.  The distribution of \eqref{eq:conditional_distribution} involves  parameters, but one can perform a plug-in based on the estimates $(\hat{\mu}, \hat{\boldsymbol{\gamma}}, \hat{\nu}, \hat{\theta_0}, \hat{\vartheta})$ from  Algorithm \ref{alg:MYALG} to obtain $\hat{ \mu}(\boldsymbol{x}_* ) $ and  $\hat{\sigma}^2(\boldsymbol{x}_* )$.

\begin{algorithm}[H]
\caption{Robust Local Gaussian Process Estimation }
\begin{algorithmic}[1]
  \State \textbf{Input:}  $\boldsymbol{x}_*$ (test location), $\mathcal{D}_n$ (nearest neighbors of  $\boldsymbol{x}_*$), $q$ (an upper bound for the number of outliers).
    \State \textbf{Output:}  $\hat{\mu}, \hat{\boldsymbol{\gamma}}, \hat{\nu}, \hat{\theta_0}$ and $\hat{\vartheta}$
        \State \textbf{Initialization:}
    \State \quad  $\gamma^{(0)} = \mathbf{0}, \mu^{(0)} = \text{Med}(\mathbf{y})$, $\nu^{(0)} = [1.483 \cdot \text{Med}(\mathbf{y} - \mu^{(0)})]^2$
    \State \quad  $\theta_0^{(0)} = \vartheta^{(0)} =1, \mathbf{S}^{(0)} = \nu^{(0)} \mathbf{I}$

    \State $t \gets 0$

    \While{not converged}
        \State $\boldsymbol {T}^{(t)} \gets (\boldsymbol {S}^{(t)})^{-1}$ 
        \State $\boldsymbol {\gamma}^{(t,0)} \gets \boldsymbol {\gamma}^{(t)}$,  $\boldsymbol {r} \gets \mathbf{T}^{(t)}(\boldsymbol {y} - \mathbf{1}\,\mu^{(t)})$, $\rho  \gets \|\boldsymbol {T}^{(t)} \|_2$
        \State  $j \gets 0$
        \While{not converged}
            \State $\boldsymbol {\gamma}^{(t,j+1)} \gets \Theta^\# (\boldsymbol { \gamma}^{(t,j)} - \frac{1}{\rho}(\boldsymbol {T}^{(t)} \boldsymbol {\gamma}^{(t,j)} - \boldsymbol {r}); q)$
           \State $j \gets j + 1$
           \EndWhile
        \State $\boldsymbol { \gamma}^{(t+1)} \gets \boldsymbol { \gamma}^{(t,j)}$
        \State Compute $\mu^{(t+1)} \gets \frac{\mathbf{1}^T \boldsymbol {T}^{(t)} (\boldsymbol {y} - \boldsymbol { \gamma}^{(t+1)})}{\mathbf{1}^T \boldsymbol {T}^{(t)} \mathbf{1}}$
        \State Compute $\nabla_{\boldsymbol{\chi}} l(\mu^{(t+1)},\boldsymbol { \gamma}^{(t+1)},\boldsymbol{\chi}^{(t)})$ and update $\boldsymbol{\chi}^{(t+1)}$ using quasi-Newton.
        \State  Using $\boldsymbol{\chi}^{(t+1)}$,  form $\boldsymbol {S}^{(t+1)}$  according to \eqref{eq:label15}.
        \State $t \gets t + 1$
    \EndWhile

  \State     $\hat{\mu} \gets \mu^{(t)}, \hat{\boldsymbol{\gamma}} \gets \boldsymbol{\gamma}^{(t)}$   and $\hat{\boldsymbol{\chi}} \gets \boldsymbol{\chi}^{(t)} $
\end{algorithmic}
\label{alg:MYALG}
\end{algorithm}

\paragraph*{Parameter Tuning}
In contrast to other algorithms  \citep{luo2021bayesian,pope2021gaussian,gramacy2008bayesian,liu2015kinect}, Algorithm \ref{alg:MYALG} has only one regularization parameter, $q$, which serves as an upper bound of the total number of outliers.  Specifically, the parameter  controls the sparsity  level of the outlyingness vector \( \boldsymbol{\gamma} \), determining the proportion of observations identified as outliers. A practical choice  sets \( q= \alpha \cdot n \). In most of our experiments, we set \( \alpha = 0.15 \), meaning  at most 15\% of the nearest neighbors are treated as potential outliers.  This choice of  \( q \)  has consistently delivered robust performance across various neighborhood configurations and dimensions in our experiments.

We can also make a more data-adaptive choice of \( q \) by employing Tukey's method on median absolute deviation (MAD), which robustly identifies outliers without relying on specific distributional assumptions. Define the confidence interval (CI) as follows:
\begin{equation}
CI = \text{Med}(\boldsymbol{y}) \pm \tau \times \text{MAD}(\boldsymbol{y}),
\end{equation}
where  \( \text{Med}(\boldsymbol{y}) \) represents the median of the data,    \( \text{MAD}(\boldsymbol{y}) = 1.483\cdot \text{Med}(|y_i - \text{Med}(\boldsymbol{y})|) \) \citep{huber1981robust}   calculates the median deviation from the median, and   $ \tau=3$  in our experiments.   We then set \( q \) equal to   the number of data points outside this interval. This adaptive approach adjusts  \( q \) to the data structure and  ensures robust performance across  various scenarios.
\section{Experiments}\label{sec:methodology}

We first study how the trimming level $q$ affects RLGP's predictive accuracy in   representative test scenarios. We then compare RLGP with 9 alternative  benchmark  methods on four real-world datasets that exhibit sharp jumps and discontinuities. Last, we test RLGP's scalability on synthetic problems of increasing dimensionality.

\subsection{Exploration of Parameter Choices}
To evaluate the impact of the trimming level parameter \(q\) on prediction accuracy, we use 2-D synthetic datasets simulated on a dense grid of \([-0.5, 0.5] \times [-0.5, 0.5]\), with the grayscale shading indicating the underlying response surface. Three test scenarios are considered as shown in Figure~\ref{fig:test_functions}.

\begin{figure}[htbp]
    \centering
    \begin{subfigure}[b]{0.31\textwidth}
        \includegraphics[width=\linewidth]{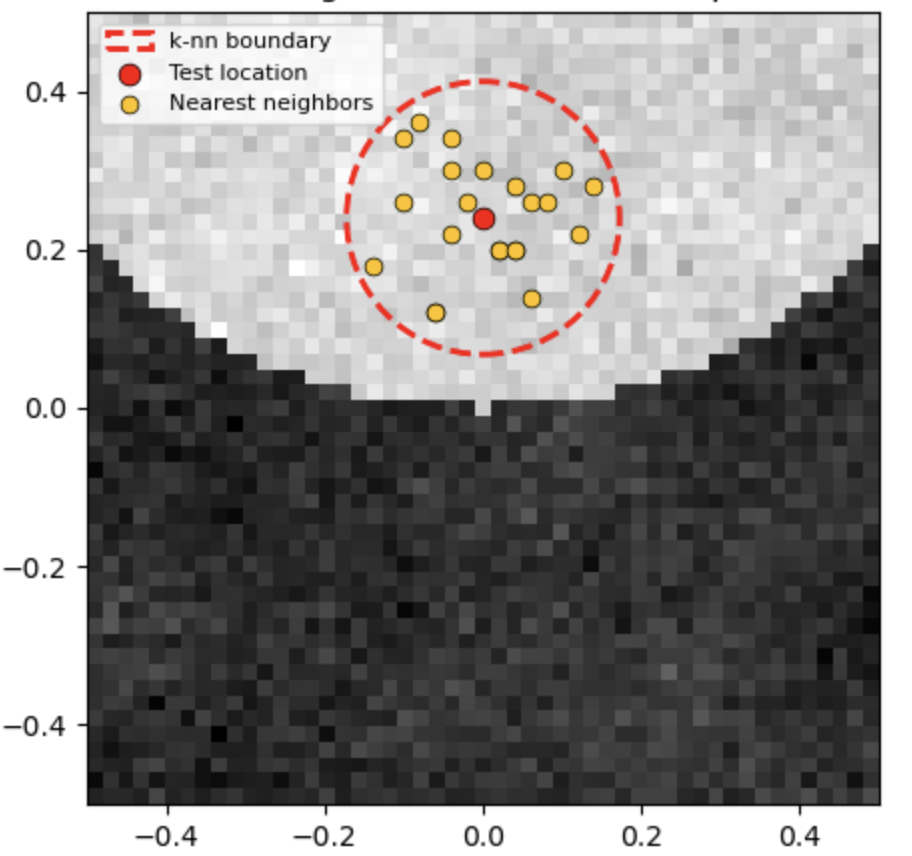}
    \end{subfigure}\hfill
    \begin{subfigure}[b]{0.31\textwidth}
        \includegraphics[width=\linewidth]{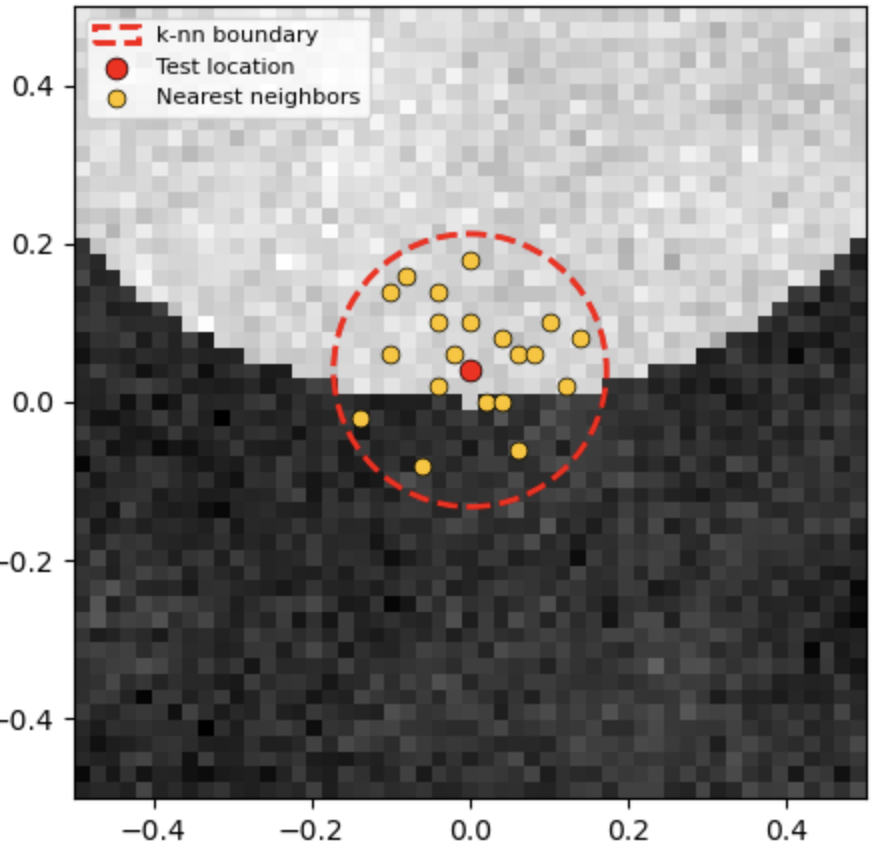}
    \end{subfigure}\hfill
    \begin{subfigure}[b]{0.31\textwidth}
        \includegraphics[width=\linewidth]{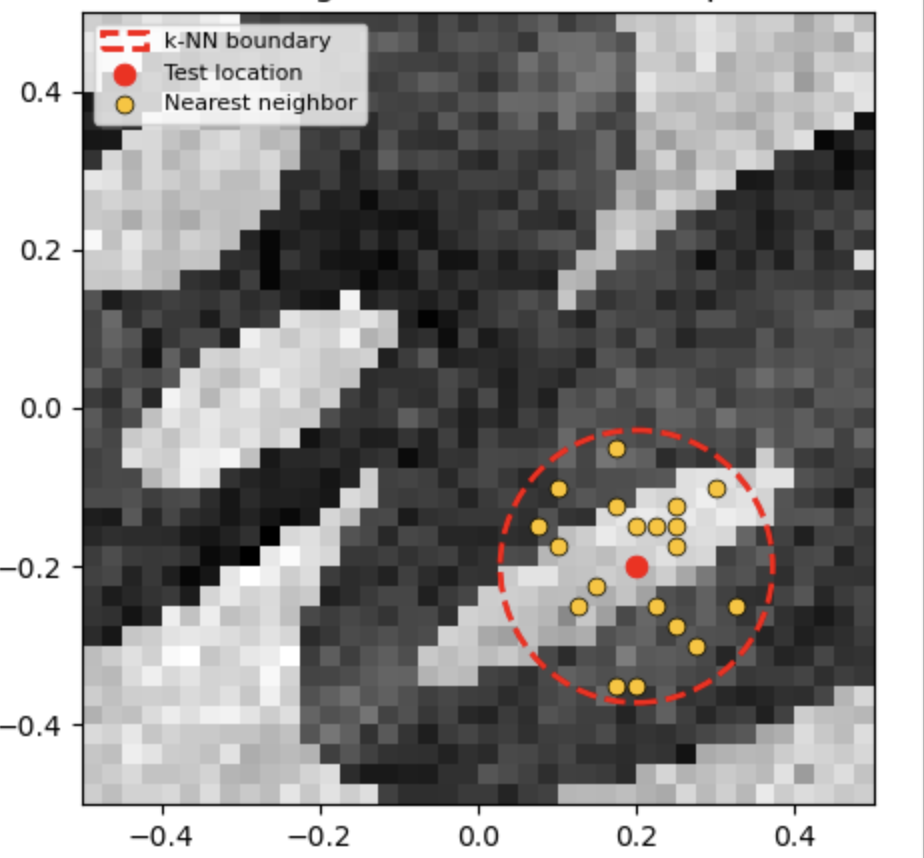}
    \end{subfigure}

    \vspace{0.75em}

    \begin{subfigure}[b]{0.31\textwidth}
        \includegraphics[width=\linewidth]{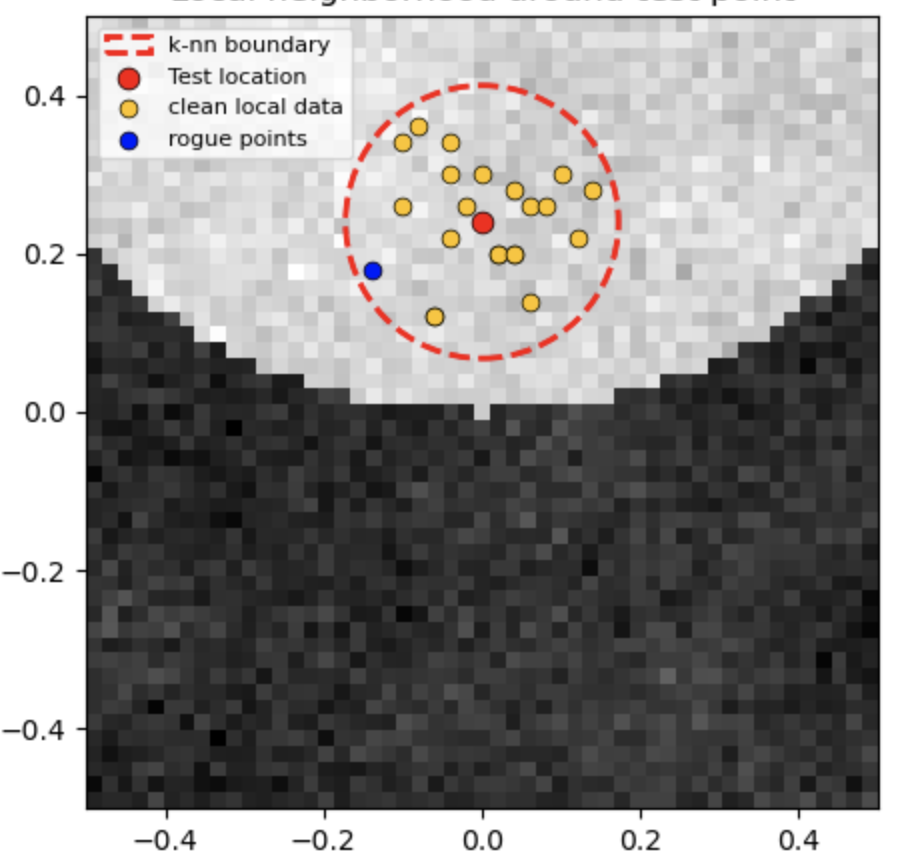}
    \end{subfigure}\hfill
    \begin{subfigure}[b]{0.31\textwidth}
        \includegraphics[width=\linewidth]{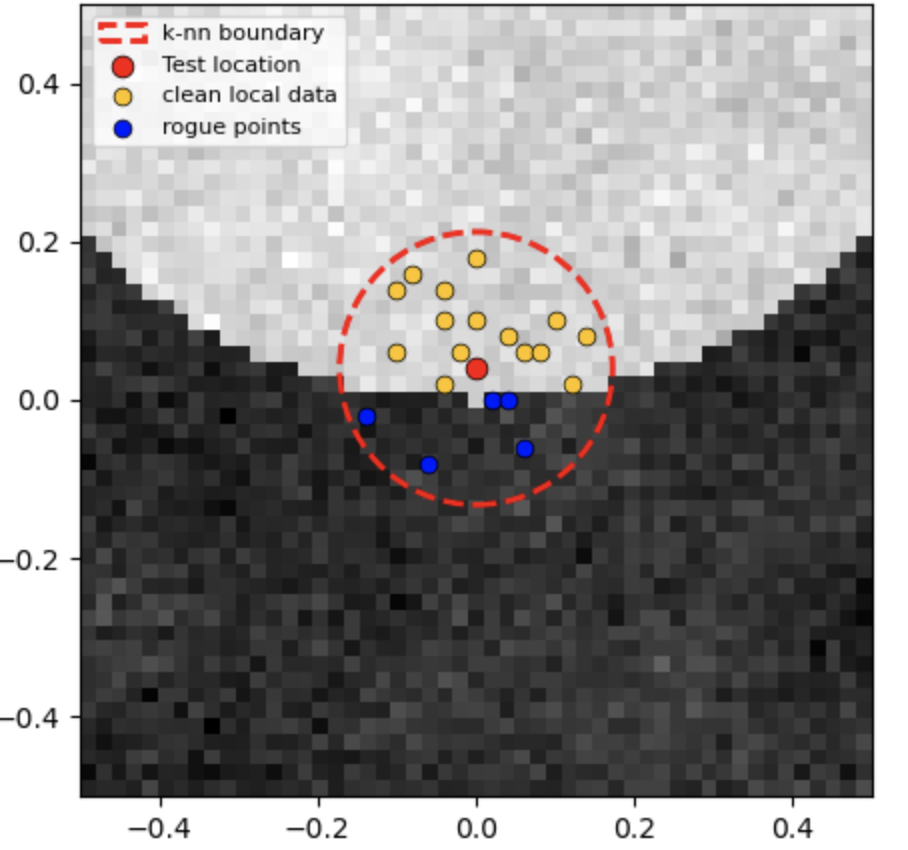}
    \end{subfigure}\hfill
    \begin{subfigure}[b]{0.31\textwidth}
        \includegraphics[width=\linewidth]{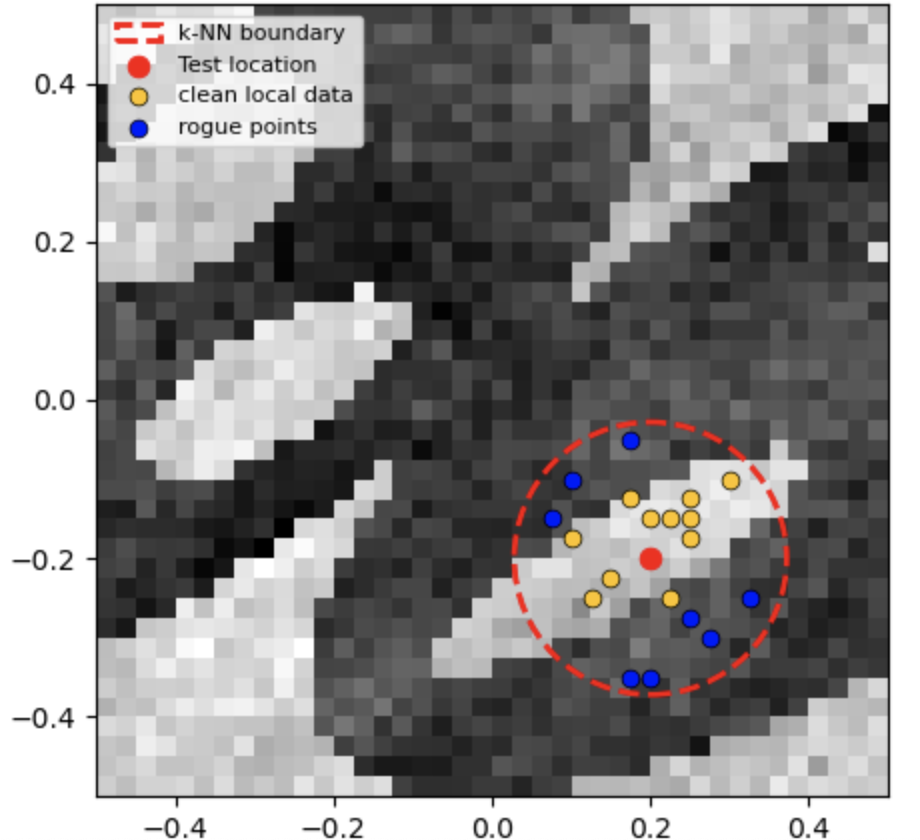}
    \end{subfigure}

    \vspace{0.75em}

    \begin{subfigure}[b]{0.31\textwidth}
        \includegraphics[width=\linewidth]{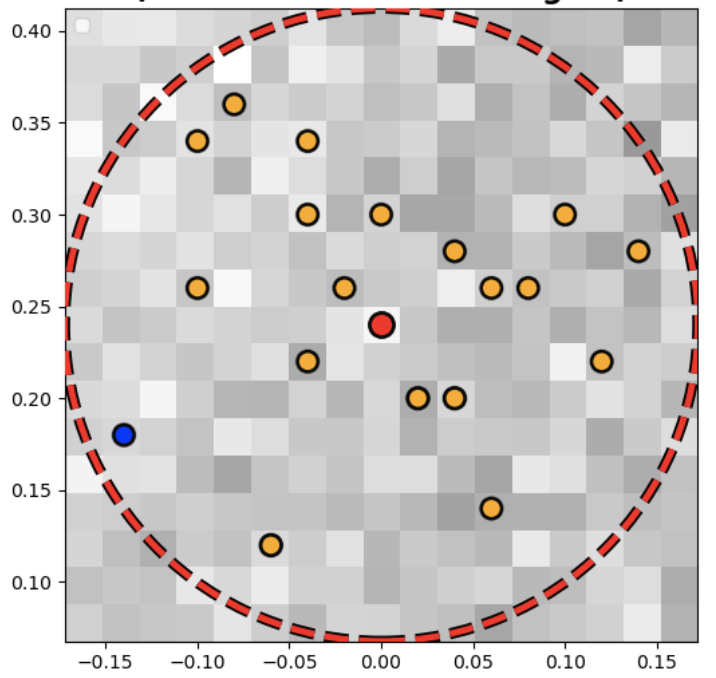}
    \end{subfigure}\hfill
    \begin{subfigure}[b]{0.31\textwidth}
        \includegraphics[width=\linewidth]{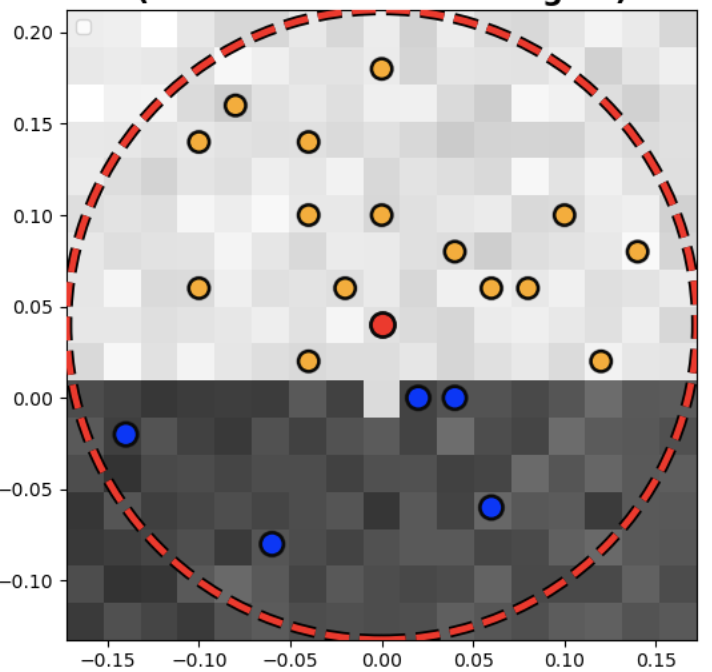}
    \end{subfigure}\hfill
    \begin{subfigure}[b]{0.31\textwidth}
        \includegraphics[width=\linewidth]{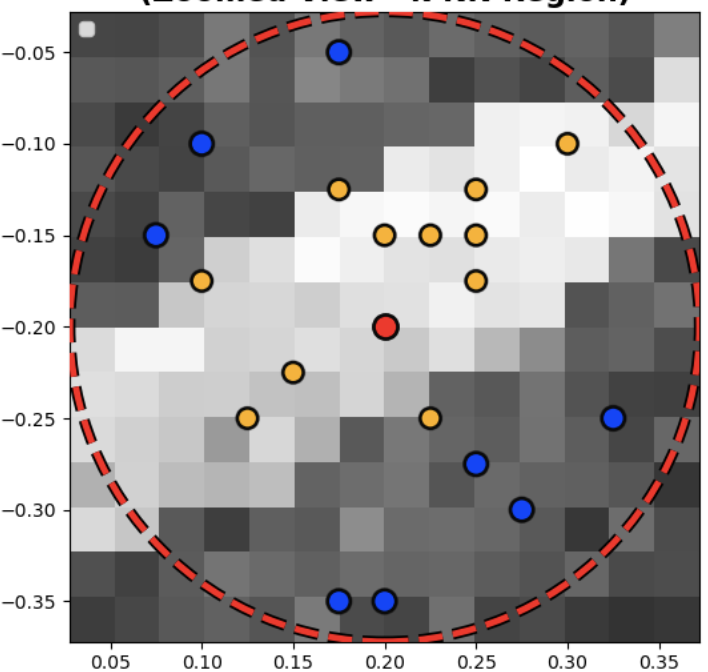}
    \end{subfigure}

 \caption{Illustration of test scenarios and the role of the adaptive-$q$ procedure.
\textit{Top row:} Nearest-neighbor data selection for the test point under three scenarios—\textit{interior test point} (left), \textit{near-simple-boundary test point} (middle), and \textit{near-complex-boundary test point} (right).
\textit{Second row:} In the same three scenarios, adaptive-$q$ mechanism separates clean local data (neighbors consistent with the test region) from rogue points or outliers, thereby improving neighborhood quality for reliable prediction.
\textit{Bottom row:} Zoomed-in view showing the resulting  neighborhood. }

    \label{fig:test_functions}
\end{figure}

In the \textit{interior test point} scenario, shown in the left column of Figure~\ref{fig:test_functions}, the test point is positioned well within the interior of the region of interest, away from its boundaries. This setting assumes a \textit{homogeneous environment} with minimal variability in the data and few, if any, outliers.
In contrast, the \textit{near-simple-boundary test point} scenario (middle column) places the test point close to the boundary of the region of interest. Such proximity often introduces heterogeneity due to the influence of adjacent regions or differing conditions at the edges. This scenario poses challenges because the transition in data characteristics near boundaries can lead to higher prediction errors, requiring more careful modeling to capture these dynamics.
Finally, in the \textit{near-complex-boundary test point} scenario, shown in the right column of Figure~\ref{fig:test_functions}, the test point is surrounded by multiple boundaries. This setting presents the greatest difficulty: the intricate structure near complex boundaries often produces rogue points or spurious neighbors, making it harder to extract a stable local neighborhood.

Table 1 shows the accuracy of our algorithm under different choices of $q$ across the three test scenarios. The optimal choice of $q$ depends on local data heterogeneity, as it governs the \textbf{bias-variance trade-off}. For clean, interior regions, a small $q$ (e.g., $q=10\%$) is most effective, as preserving a larger, informative neighborhood minimizes variance. In contrast, for points near boundaries, a larger $q$ is necessary to trim the neighborhood by removing outliers, which reduces bias from contaminating data. On the other hand, a key observation from our experiments is that RLGP is not overly sensitive to the precise choice of $q$. Even when $q$ is set slightly higher or lower than optimal, the prediction accuracy and uncertainty calibration remain stable. This robustness is helpful in practice, since it ensures that RLGP can be deployed without exhaustive tuning of $q$, making the method well-suited for automated modeling pipelines in industrial applications where data are heterogeneous and discontinuous.

Our adaptive $q$ strategy automates this process, providing superior accuracy by adjusting to local data variations in a data-dependent manner.
Figure~\ref{fig:test_functions} also illustrates this process. The top row shows that nearest-neighbor selection alone may inadvertently include points from outside the true region of influence, particularly near boundaries. The middle row and bottom row demonstrate how adaptive-$q$ excludes such rogue points (labeled by blue), yielding neighborhoods that are both locally coherent and robust. This selective filtering explains why adaptive-$q$ consistently achieves the best mean squared error.
This flexibility is particularly effective in real-world, high-dimensional datasets possibly characterized by jumps, discontinuities, and conflicting patterns. Below we apply the adaptive choice of $q$ by default.

\begin{table*}[ht!]
\centering
\renewcommand{\arraystretch}{1.2}
\begin{tabular*}{\textwidth}{@{\extracolsep{\fill}}ccccc@{}}
\toprule
$q$& Interior Test Point & Simple Boundary& Complex Boundary  \\
\midrule
$10\%n$ &  0.30 & 1.80& 2.85\\
$15\%n$ &  0.32  & 1.43 &  2.14  \\
$20\%n$&  0.41& 0.18& 1.03 \\
$30\%n$ &  0.42& 0.11& 0.35  \\
\textbf{Adaptive}  &   \textbf{0.27} &\textbf{ 0.15}&\textbf{0.31}\\
\bottomrule
\end{tabular*}
\caption{Absolute prediction errors for different trimming levels $q$ across interior, simple-boundary, and complex-boundary test scenarios.}

\label{tab:combined-results}
\end{table*}

\subsection{Real-World Benchmark Case Studies}
\subsubsection{Engineering Background and Data Characteristics}
The evaluation uses four real-world datasets, with response surfaces marked by discontinuities and abrupt transitions. These datasets cover materials science, image reconstruction, environmental monitoring, and computational biology.\footnote{Data sources: The Nanotube, CSImage, and Corrosion datasets are available at \url{https://drive.google.com/file/d/1XLQTd0XdqQPQ3f5jLM1aJsb3eL2lO5Eh/view}. The Cancer dataset (CCLE 2019 RPPA) is available at \url{https://depmap.org/portal/data_page/?tab=allData&releasename=CCLE\%202019&filename=CCLE_RPPA_20181003.csv}.  } Their sharp transitions pose significant challenges to traditional modeling approaches. We now outline the scientific context and key characteristics of each dataset.

\paragraph{Nanotube}
A motivating case study of this research work is to predict the response of a chemical synthesis experiment under specified experimental conditions, where the measured output exhibits abrupt changes across certain characteristic boundaries.

Carbon nanotubes (CNTs) are cylindrical nanostructures composed of carbon atoms, celebrated for their exceptional tensile strength, electrical conductivity, and thermal stability. These properties make CNTs indispensable in high-performance composites, nanoelectronics, sensors, and energy storage systems. In industrial settings, CNTs are typically synthesized via catalytic chemical vapor deposition (CVD) \citep{magrez2010catalytic}, where precise control over process conditions is essential for maximizing yield.
Recent advances in CNT manufacturing have introduced fully automated robotic CVD platforms capable of controlling numerous parameters—including gas flow rates, catalyst types, and promoter concentrations—while performing in-situ Raman spectroscopy for real-time yield measurement. These systems enable rapid data acquisition across a wide parameter space with minimal human intervention. Nonetheless, each experimental run remains costly, and constrained budgets limit the number of feasible trials—especially when yield behavior includes abrupt, localized transitions.
Such abrupt changes often arise due to catalyst phase transitions. Two key parameters dominate CNT yield: (1) reaction temperature, and (2) the concentration ratio of the growth catalyst (C$_2$H$_4$) to the growth suppressor (CO$_2$). Experimental evidence reveals that yield remains near zero across wide regions, but slight shifts in these parameters can trigger sudden transitions to a high-yield plateau. These sharp discontinuities result in complex response surfaces that pose significant challenges for conventional smooth surrogate models such as standard Gaussian Processes, which tend to blur over such transitions and fail to accurately predict behavior near critical boundaries.

The CNT yield prediction task sits within a broader \textit{three-stage} optimization pipeline. In Stage~1 (Initial Design), a limited number of experimental configurations are selected and tested. In Stage~2 (Modeling), the focus of this work, a surrogate model is trained on the fixed dataset to capture discontinuities and handle heterogeneous, noisy responses. In Stage~3 (Active Learning / Sequential Design), the trained model is used to iteratively guide new experiments and refine knowledge near transition boundaries.
The fidelity of Stage~2 is critical to the overall pipeline. Inaccurate modeling near discontinuities can lead to suboptimal experiment selection and wasted resources. Reliable modeling of yield discontinuities enables more effective optimization by allowing the system to operate close to optimal catalyst ratios and temperatures without overshooting into suboptimal regimes. It helps reduce waste by avoiding experimental conditions that would likely result in poor outcomes. Furthermore, robust surrogate models support real-time process control through adaptive feedback and empower closed-loop systems where the model itself guides future experimental decisions. This not only accelerates scientific discovery but also significantly reduces costs.

The dataset used in this study, denoted as \textbf{Nanotube}, consists of 52 experimental observations collected under varying CVD conditions. The inputs are reaction temperature and the logarithmic ratio of two chemical reactants (C$_2$H$_4$ and CO$_2$), while the output is the measured nanotube yield. Yield values in this dataset exhibit abrupt transitions, especially near catalyst activation thresholds. These localized changes create sharp nonlinearities that are difficult for traditional surrogate models to capture. Given the limited number of observations and the high cost of data acquisition, we assess model performance using leave-one-out cross-validation (LOOCV), ensuring maximum utilization of available information and robust error estimation under small-sample constraints.

\paragraph{CSImage} Electron microscopy is a critical technology in materials science, enabling high-resolution imaging of atomic and nanoscale structures. However, raster-scanning every pixel of a high-resolution image is time-intensive and can expose specimens to excessive electron doses, potentially damaging sensitive materials. To address this, compressive sensing strategies are increasingly used. These approaches acquire measurements at a carefully selected subset of spatial locations, reducing scan time while maintaining informative content.

Our case study aims to develop a predictive model that can accurately reconstruct the underlying intensity surface from non-uniform data while providing well-calibrated uncertainty estimates. The surface here has sharp discontinuities at material boundaries, which violates the smoothness assumptions of many conventional surrogate models. Furthermore, the heterogeneous sampling of the dataset, with observations concentrated in complex regions and sparse elsewhere, requires a model that can adapt to varying data densities. The problem is also characterized by high-dimensional uncertainty, which means the model must effectively capture both structural noise from measurements and epistemic uncertainty in poorly observed areas.

The dataset used in this study, denoted as \textbf{CSImage}, was obtained from such a compressive sensing protocol. Here, the electron beam adaptively sampled only the most informative regions, guided by an automated design strategy. The result is a dataset consisting of two-dimensional spatial coordinates and their associated electron intensity responses. Crucially, the spatial observations are non-uniformly distributed: they are denser in regions rich with boundaries—such as those between particle agglomerates and background substrate—and sparser in homogenous regions. This sampling pattern creates a highly structured yet irregular dataset that is representative of real-world imaging pipelines. The dataset contains 17,519 sensing points selected through an adaptive algorithm. The inputs are spatial coordinates in two dimensions, while the outputs are intensity measurements from an electron microscope. These measurements exhibit discontinuities at material boundaries, such as the interfaces between agglomerates and surrounding substrate. For modeling purposes, the data are divided into 90\% training and 10\% testing sets. 

\paragraph{Corrosion}

   Environmental corrosion of metallic components is a critical concern in sectors such as aerospace, maritime operations, and infrastructure. Corrosion rates are highly sensitive to a range of environmental factors including temperature, humidity, pollutant concentrations, and rainfall events. Monitoring corrosion in situ requires the use of field-deployed sensors that capture both environmental variables and electrochemical responses, such as corrosion current.
 When analyzing environmental corrosion data collected from long-term sensor monitoring experiments, the  goal is often to accurately predict corrosion currents under varying atmospheric conditions, especially near threshold regimes where corrosion behavior changes abruptly.

In this case study, we analyze a real-world dataset comprising sensor measurements collected under natural outdoor conditions. Over several months, sensors recorded environmental variables such as air temperature, surface temperature, relative humidity, and electrochemical impedance, along with corresponding galvanic corrosion current measurements on metallic specimens. The dataset used in this study, denoted as \textbf{Corrosion}, includes 10,153 representative sensor readings collected under varying environmental conditions. The input variables include air temperature, surface temperature, relative humidity, and effective humidity, while the response is the corrosion current measured on metallic specimens. These measurements often exhibit sharp changes near threshold conditions that arise from environmental triggers such as sudden humidity spikes or dew point crossings. As with the CSImage case, the data are divided into 90\% training and 10\% testing subsets.

The prediction task is non-trivial due to several data-specific challenges. First, the corrosion response surface exhibits sharp nonlinearities near threshold regions—such as high humidity or temperature crossover points—violating the smoothness assumptions of many standard surrogate models. Second, measurements are affected by varying noise levels across different environmental conditions, making heteroskedasticity an important consideration. Finally, certain regions of the input space are underrepresented, necessitating well-calibrated uncertainty estimates to avoid overconfident extrapolations.

\paragraph*{Cancer}
Cancer cell lines serve as controlled experimental models for understanding tumor biology and therapeutic responses. The Cancer Cell Line Encyclopedia (CCLE) project provides a comprehensive multi-omic resource encompassing genomic, transcriptomic, and proteomic profiles across a diverse set of human cancers \citep{ghandi2019next}. However, mapping the complex relationships between signaling components in oncogenic pathways is challenging due to highly nonlinear interactions and abrupt context-dependent shifts that drive heterogeneous phenotypic outcomes across tumor types.

Our case study aims to develop a predictive model that can accurately reconstruct protein phosphorylation levels from molecular feature data while providing well-calibrated uncertainty estimates. The signaling relationships here exhibit sharp discontinuities at pathway boundaries, which violates the smoothness assumptions of many conventional surrogate models. Furthermore, the heterogeneous nature of cancer lineages, with distinct molecular profiles across different tumor types, requires a model that can adapt to varying biological contexts.

The dataset used in this study, denoted as \textbf{Cancer}, was obtained from the Reverse Phase Protein Array   component of the CCLE project. This dataset quantifies protein and phospho-protein abundances for cancer cell lines across multiple signaling features, capturing complex interactions in oncogenic pathways such as PI3K–Akt–mTOR, MAPK, and Wnt. The dataset contains protein expression measurements from 899 cancer cell lines. The inputs are 6 key upstream and pathway-related proteins identified through feature selection: EGFR, HER2, PI3K-p110-alpha\_Caution, PTEN, MAPK\_pT202\_Y204, and Cyclin\_D1, while the response is the phosphorylation level of Akt\_pS473. These measurements exhibit discontinuities at signaling pathway boundaries, such as the transitions between different regulatory states in oncogenic cascades. The data are divided into 90\% training and 10\% testing sets. The dataset and annotations follow the CCLE 2019 release \citep{ghandi2019next}, with related metabolomic profiles described in \citep{li2019landscape}.
\subsubsection{Methods for Comparison and Implementation}
We evaluate the performance of  several benchmarks in addition to RLGP: Bayesian Treed Gaussian Process (TGP, \citealt{gramacy2008bayesian}), Local Gaussian Process (Local GP, \citealt{nguyen2009model}), Dynamic Tree Model (DynaTree, \citealt{taddy2011dynamic}), Local Approximate Gaussian Process (laGP, \citealt{gramacy2015local}), Locally Induced Gaussian Process (liGP, \citealt{cole2021locally}), and Jump Gaussian Process (JGP, \citealt{park2022jump}). We refer to JGP with a linear partitioning function as JGP-L and JGP with a quadratic partitioning function as JGP-Q.  Additionally, we include two deep learning models, Bayesian Neural Networks (BNNs, \cite{jospin2022hands}) and Deep Gaussian Processes (DeepGPs, \cite{damianou2013deep}).

The methods we are considering fall into three main categories: tree-based,  nearest-neighbor-based, and neural-net-based.  Tree-based models, such as TGP and DynaTree, use axis-aligned recursive partitioning to segment  the input space into disjoint regions, fitting independent models within each partition. TGP applies a GP to each region, allowing flexibility but increasing computational costs. DynaTree, in contrast, fits simpler constant or linear models within partitions, making it computationally more efficient but limiting its ability to capture nonlinear response dynamics. While both models can effectively handle abrupt changes in response surfaces, their reliance on predefined partitioning can lead to over-segmentation in high-dimensional settings, making them less adaptable to complex boundary structures.

Nearest-neighbor-based models, including Local GP, laGP, liGP, JGP-L, JGP-Q, and RLGP, construct local training subsets by selecting nearest neighbors around the test point. Local GP simply fits a GP to a fixed set of nearest neighbors of size \( n  \). laGP refines this approach by starting with a small set of neighbors and iteratively expanding the subset based on a mean squared predictive error  criterion, optimizing data selection for improved predictive accuracy. liGP takes a different approach by reducing the subset through a selection of inducing points to construct a more compact local GP model. JGP further segments the nearest-neighbor subset by applying a partitioning function either a linear hyperplane (JGP-L) or a quadratic function (JGP-Q) to divide the data into two regions and fit a GP to the subset containing the test point. RLGP, the proposed model, selects a nearest-neighbor subset of size \( n \) and employs a robust optimization framework to identify and mitigate outliers before fitting a GP to the refined  subset of size $m \le n$ by adaptively handling outliers.

Unlike the previous methods, probabilistic deep models learn hierarchical data representations through multi-layered architectures. BNNs extend traditional neural networks by placing a probability distribution over the network's weights. By treating weights as random variables, BNNs can capture epistemic uncertainty, providing a more robust measure of confidence in their output.
DeepGPs are a multi-layered extension of traditional GPs. By composing multiple GP layers, they create a hierarchical structure where the output of one layer serves as the input to the next.  This allows them to effectively model highly non-stationary functions and learn complex, non-linear mappings.

We employ standard MATLAB implementations for JGP and Local GP, while laGP, liGP, TGP, and DynaTree are used via R packages provided by their respective authors. Notably, laGP is implemented in C and uses OpenMP for \textit{parallelization}, while both TGP and DynaTree are built with a combination of C and C++. The liGP model also supports parallel processing. DeepGP and BNN are implemented through Python libraries. Our RLGP is developed in Python. Although RLGP could be implemented in more efficient, lower-level programming languages such as C++ or C, and further optimized with parallelization, our current non-parallel Python version already matches or exceeds state-of-the-art methods in both execution time and accuracy, as demonstrated by the experimental results in the next subsection.
All analyses were performed using a MacBook Pro (2017) with a 3.5 GHz dual-core Intel Core i7 processor and 16 GB of 2133 MHz LPDDR3 memory.

\subsubsection{Comparative Analysis}
\label{sec:comparative_analysis}
We evaluated each method using three metrics: Mean Squared Error (MSE) to measure point-prediction accuracy, the Continuous Ranked Probability Score (CRPS) to assess the overall quality of the predictive distribution, and the average computational time per test point (in seconds). For all three metrics, smaller values indicate better performance.


Table \ref{tab:expanded_performance_comparison} summarizes the performance of all 10 methods across four real-world datasets, evaluating them on prediction accuracy, uncertainty quantification, and computational efficiency. To highlight the benefits of our approach, Table \ref{tab:rel_improvement_rlgp_methods_as_rows} details the relative performance gains of RLGP over the 9 baseline methods.
Since prediction accuracy is often the most critical metric, Figure~\ref{fig:mse_real_data} provides a focused visualization of the MSE results to offer a clearer intuition of our method's effectiveness.

\begin{table*}[ht!]
\centering
\renewcommand{\arraystretch}{1.3}
\setlength{\tabcolsep}{6pt}

\begin{tabular}{@{}lcccccc@{}}
\toprule
 & \multicolumn{3}{c}{\textbf{Nanotube Dataset}} & \multicolumn{3}{c}{\textbf{CSImage Dataset}} \\
\cmidrule(lr){2-4} \cmidrule(lr){5-7}
 & \textbf{MSE} & \textbf{CRPS} & \textbf{Time} & \textbf{MSE} & \textbf{CRPS} & \textbf{Time} \\
\midrule
TGP      & 1.10 & 0.50 & 0.37 & 0.18 & 0.49 & 0.92 \\
DynaTree & 1.05 & 0.51 & 0.13 & 2.71 & 0.59 & 0.36 \\
laGP     & 1.11 & 0.55 & 0.14 & 1.67 & 0.69 & 0.20 \\
liGP     & 1.08 & 0.51 & 0.29 & 0.67 & 0.53 & 0.18 \\
LocalGP  & 1.32 & 0.58 & 0.29 & 0.20 & {0.10} & 0.09 \\ 
JGP-L    & 1.10 & 0.49 & 0.60 & 0.14 & 0.62 & 0.31 \\
JGP-Q    & 1.38 & 0.56 & 0.86 & 0.14 & 0.60 & 0.82 \\
DeepGP   & 1.48 & 0.43 & 0.98 & {1.99} & {0.64} & 0.36 \\ 
BNN      & 1.51 & 0.68 & 0.32 & {0.26} & {0.31} & 0.03 \\ 
RLGP     & 1.05 & 0.08 & 0.32 & 0.14 & {0.29} & 0.41 \\ 
\midrule[0.8pt]

 & \multicolumn{3}{c}{\textbf{Corrosion Dataset}} & \multicolumn{3}{c}{\textbf{Cancer Dataset}} \\
\cmidrule(lr){2-4} \cmidrule(lr){5-7}
 & \textbf{MSE} & \textbf{CRPS} & \textbf{Time} & \textbf{MSE} & \textbf{CRPS} & \textbf{Time} \\
\midrule
TGP      & 1.17 & 0.61 & 0.59 & 1.25 & 0.61 & 11.62 \\
DynaTree & 1.30 & 0.49 & 1.09 & 1.52 & 0.68 & 0.08 \\
laGP     & 0.99 & 0.46 & 0.28 & 1.25 & 0.62 & 0.05 \\
liGP     & 1.21 & 0.63 & 0.16 & 1.47 & 0.78 & 0.08 \\
LocalGP  & 0.92 & 0.45 & 0.22 & 1.39 & 0.71 & 0.03 \\
JGP-L    & 0.80 & 0.43 & 0.47 & 1.37 & 0.72 & 0.17 \\
JGP-Q    & 0.78 & 0.43 & 1.25 & 1.91 & 0.92 & 0.36 \\
DeepGP   & 0.66 & 0.38 & 0.58 & 1.27 & 0.62 & 0.28 \\
BNN      & {1.06} & 0.53 & 0.25 & 1.32 & 0.63 & 0.03 \\ 
RLGP     & 0.68 & {0.41} & 0.30 & 1.33 & 0.63 & 0.05 \\ 
\bottomrule
\end{tabular}

\caption{Performance comparison of 10 methods on the Nanotube, CSImage, Corrosion, and Cancer datasets. Models are evaluated using Mean Squared Error (MSE), Continuous Ranked Probability Score (CRPS), and average computation time per test point (in seconds).}
\label{tab:expanded_performance_comparison}
\end{table*}

\begin{table*}[ht!]
\centering
\renewcommand{\arraystretch}{1.15}
\setlength{\tabcolsep}{6pt}
\begin{tabular}{@{}lcccccccc@{}}
\toprule
 & \multicolumn{2}{c}{\textbf{Nanotube}}
 & \multicolumn{2}{c}{\textbf{CSImage}}
 & \multicolumn{2}{c}{\textbf{Corrosion}}
 & \multicolumn{2}{c}{\textbf{Cancer}} \\
\cmidrule(lr){2-3} \cmidrule(lr){4-5} \cmidrule(lr){6-7} \cmidrule(lr){8-9}
 & MSE & CRPS & MSE & CRPS & MSE & CRPS & MSE & CRPS \\
\midrule
TGP      & +4.5\%  & +84.0\% & +22.2\% & +40.8\% & +41.9\% & +32.8\% & -6.0\%  & -3.2\%  \\
DynaTree & +0.0\%  & +84.3\% & +94.8\% & +50.8\% & +47.7\% & +16.3\% & +14.3\% & +7.9\%  \\
laGP     & +5.4\%  & +85.5\% & +91.6\% & +58.0\% & +31.3\% & +10.9\% & -6.0\%  & -1.6\%  \\
liGP     & +2.8\%  & +84.3\% & +79.1\% & +45.3\% & +43.8\% & +34.9\% & +10.5\% & +23.8\% \\
LocalGP  & +20.5\% & +86.2\% & +30.0\% & +51.7\% & +26.1\% & +8.9\%  & +4.5\%  & +12.7\% \\
JGP-L    & +4.5\%  & +83.7\% & 0.0\%   & +53.2\% & +15.0\% & +4.7\%  & +3.0\%  & +14.3\% \\
JGP-Q    & +23.9\% & +85.7\% & 0.0\%   & +51.7\% & +12.8\% & +4.7\%  & +43.6\% & +46.0\% \\
DeepGP   & +11.0\% & +81.4\% & +51.7\% & +9.4\%  & -3.0\%  & -7.9\%  & -4.5\%  & -1.6\%  \\
BNN      & +54.3\% & +88.2\% & +50.0\% & +58.0\% & +35.8\% & +22.6\% & -0.8\%  & +0.0\%  \\
\bottomrule
\end{tabular}
\caption{Relative increase (\%) in error (MSE and CRPS) of the competing methods, using RLGP as the baseline. Positive values indicate the competitor had a higher error, thus demonstrating a performance advantage for RLGP.}
\label{tab:rel_improvement_rlgp_methods_as_rows}
\end{table*}

\begin{figure}[ht!]
    \centering
    \includegraphics[width=\textwidth]{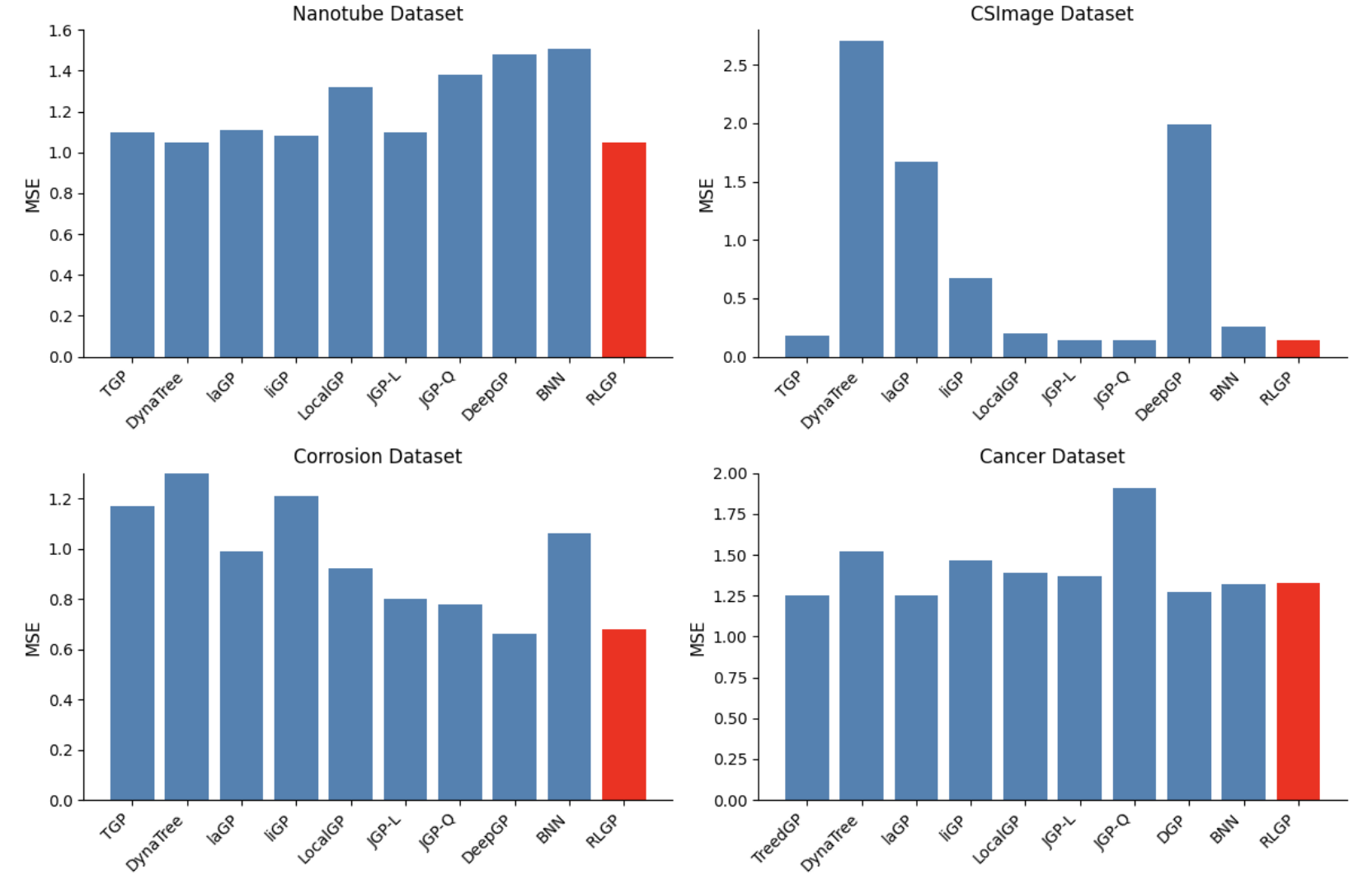}
    \caption{Illustration of prediction accuracy (MSE) on the Nanotube, CSImage, and Corrosion datasets. RLGP, highlighted in red, consistently delivers the best performance (lower is better).}
    \label{fig:mse_real_data}
\end{figure}

According to  Table \ref{tab:expanded_performance_comparison}, laGP, liGP, and Local GP are computationally efficient across all datasets. However, this efficiency often comes at the expense of predictive accuracy, particularly on Corrosion and CSImage. Based on our extensive experience, these models often struggle to capture the complex response structures present in high-dimensional and discontinuous settings.

TGP, DynaTree, and liGP exhibited particularly high MSEs on the Corrosion dataset. TGP and DynaTree use segmentation-based approaches to model complex, heterogeneous response surfaces with abrupt changes but typically did not perform effectively. Similarly, liGP, which does not use explicit partitioning and relies on a limited subset of inducing points, failed to adequately capture the underlying structure.

JGP-L and JGP-Q are competitively accurate but the most computationally costly methods. For instance, on the Nanotube dataset, their execution times are approximately double those of other models. This indicates that the added complexity from localized partitioning functions significantly boosts computational overhead and often curtails their scalability for large datasets or high-dimensional problems.

The probabilistic deep models show inconsistent performance: DeepGP achieves the best accuracy on the Corrosion dataset but is the second-worst performer on Nanotube, where BNN is the worst.
In contrast, our proposed RLGP is consistently among the top-performing methods   in terms of accuracy, and its CRPS is either the lowest or second lowest across all   datasets.

Overall, RLGP emerges as a practical and robust alternative, providing a favorable trade-off between efficiency and predictive reliability, making it particularly well-suited for datasets with nonstationary and discontinuous response structures.

\paragraph*{Practical Optimization Benefits} As suggested by a  reviewer, we now provide a more detailed discussion of our method's practical optimization benefits. We use the  Nanotube dataset as a guiding example to illustrate these advantages.
While conventional surrogate models fail here due to process noise and discontinuities, RLGP delivers tangible advantages for downstream optimization. By reliably modeling process boundaries and automatically down-weighting anomalies (e.g., from sensor drift), it provides a more faithful map of the experimental landscape. This improved accuracy, combined with well-calibrated uncertainty estimates, allows for more efficient active learning to guide the sequential design of experiments. Furthermore, its high computational efficiency enables seamless integration into live experimental workflows. In essence, RLGP provides a more robust and reliable surrogate model, directly accelerating the discovery of optimal CNT growth conditions.


\subsection{Higher-Dimensional Synthetic Datasets}\label{subsec:highdimsimu}
To evaluate the scalability and robustness of RLGP, we conducted simulations across multiple input dimensions with  $d$ ranging from 10 to 500. We draw inputs $\boldsymbol{x}$ uniformly from the rectangular domain $[-0.5,0.5]^d$ using Latin hypercube designs (LHS). The training size is fixed by the “$10d$ rule,” $n_{\text{train}}=10d$  \citep{loeppky2009choosing}; the test size is $n_{\text{test}}=1000$. The response is sampled from a two-region partitioned GP model
$
f(\boldsymbol{x}) = f_1(\boldsymbol{x})\,{1}_{\mathcal{X}_1}(\boldsymbol{x}) \;+\; f_2(\boldsymbol{x})\,{1}_{[-0.5,0.5]^d\setminus\mathcal{X}_1}(\boldsymbol{x}),
$
where $\mathcal{X}_1=\{\boldsymbol{a}^\top\boldsymbol{x}\ge 0\}$ with $\boldsymbol{a}$ chosen uniformly at random from $\{-1,1\}^d$. We take $f_1$ from a zero-mean GP with marginal variance $7$ and isotropic squared-exponential correlation with length-scale $\vartheta=0.1\,d$, and $f_2$ from a GP with mean $11$, variance $7$, and identical correlation function to  $f_1$; independent Gaussian noise with variance $3$ is added.

In this experiment, we evaluated RLGP against JGP-L, JGP-Q, DeepGP, and BNN, as these methods were the most competitive performers on the CSImage and Corrosion datasets. Other methods, such as tree-based GPs, were excluded due to either their poor performance in prior experiments or their inability to scale (for instance, even at $d=10$, they could require several hours to run and exhibited severe convergence issues).

The performance comparison is given in Table~\ref{tab:hd_results_rmse_crps_time_blocks} and  Figure  \ref{fig:mse_box_stacked}. Our experimental results highlight a clear trade-off between computational scalability and predictive accuracy among the methods.
First, JGP-L and JGP-Q demonstrated significant computational limitations: they were already outperformed and substantially slower at $d=10$, and they  became unstable and failed to scale to higher dimensions.

When evaluating computational efficiency among the scalable models, the time column shows that the two deep learning methods, DeepGP and BNN, were faster. This is an expected result, as these models are built on architectures and libraries such as GPflux (on TensorFlow), GPyTorch (on PyTorch), or TensorFlow Probability  that are inherently designed for massive parallelization. These   implementations  are heavily optimized to automatically leverage GPU acceleration for their core tensor and matrix computations.
In contrast, the proposed RLGP model was run in a strictly serial mode on a CPU, which accounts for its wall-clock times being relatively slower than the GPU-accelerated methods. However, RLGP remains highly efficient and scalable, with its computation time (e.g., in the 0.2-0.4 second range across all tested dimensions) demonstrating its practical feasibility.

 Moreover, as demonstrated in Figure~\ref{fig:mse_box_stacked} and Table~\ref{tab:hd_results_rmse_crps_time_blocks}, the computational efficiency offered by the two deep learning methods is counterbalanced by lower predictive accuracy. The MSE results clearly show that RLGP consistently outperformed both DeepGP and BNN, achieving the lowest error across the range of dimensions evaluated up to 500.

\begin{table*}[p!]
\centering
\renewcommand{\arraystretch}{1.15}
\setlength{\tabcolsep}{4pt}

{%
\begin{tabular}{@{}l*{3}{ccc}@{}}
\toprule
 & \multicolumn{3}{c}{$d=10$}
 & \multicolumn{3}{c}{$d=25$}
 & \multicolumn{3}{c}{$d=50$} \\
\cmidrule(lr){2-4}\cmidrule(lr){5-7}\cmidrule(lr){8-10}
 & \textbf{MSE} & CRPS & Time
 & \textbf{MSE} & CRPS & Time
 & \textbf{MSE} & CRPS & Time \\
\midrule
JGP-L  & \textbf{7.62} & 4.80 & 1.21 & -- & -- & -- & -- & -- & -- \\
JGP-Q  & \textbf{7.71} & 4.83 & 1.33 & -- & -- & -- & -- & -- & -- \\
DeepGP & \textbf{5.79} & 4.78 & 0.01 & \textbf{6.54} & 3.87 & 0.01 & \textbf{7.11} & 4.24 & 0.02 \\
BNN    & \textbf{7.44} & 4.90 & 0.01 & \textbf{7.08} & 4.69 & 0.05 & \textbf{7.33} & 5.05 & 0.04 \\
RLGP   & \textbf{5.57} & 3.05 & 0.26 & \textbf{5.22} & 2.87 & 0.27 & \textbf{5.96} & 3.22 & 0.26 \\
\midrule[1.5pt] 

 & \multicolumn{3}{c}{$d=75$}
 & \multicolumn{3}{c}{$d=100$}
 & \multicolumn{3}{c}{$d=150$} \\
\cmidrule(lr){2-4}\cmidrule(lr){5-7}\cmidrule(lr){8-10}
 & \textbf{MSE} & CRPS & Time
 & \textbf{MSE} & CRPS & Time
 & \textbf{MSE} & CRPS & Time \\
\midrule
JGP-L  & -- & -- & -- & -- & -- & -- & -- & -- & -- \\
JGP-Q  & -- & -- & -- & -- & -- & -- & -- & -- & -- \\
DeepGP & \textbf{7.70} & 4.63 & 0.03 & \textbf{7.18} & 4.30 & 0.05 & \textbf{6.39} & 3.93 & 0.07 \\
BNN    & \textbf{8.17} & 5.65 & 0.00 & \textbf{7.33} & 5.02 & 0.05 & \textbf{7.38} & 5.04 & 0.07 \\
RLGP   & \textbf{6.69} & 3.63 & 0.41 & \textbf{6.42} & 3.53 & 0.27 & \textbf{6.05} & 3.37 & 0.20 \\
\midrule[1.5pt] 

 & \multicolumn{3}{c}{$d=200$}
 & \multicolumn{3}{c}{$d=250$}
 & \multicolumn{3}{c}{$d=300$} \\
\cmidrule(lr){2-4}\cmidrule(lr){5-7}\cmidrule(lr){8-10}
 & \textbf{MSE} & CRPS & Time
 & \textbf{MSE} & CRPS & Time
 & \textbf{MSE} & CRPS & Time \\
\midrule
JGP-L  & -- & -- & -- & -- & -- & -- & -- & -- & -- \\
JGP-Q  & -- & -- & -- & -- & -- & -- & -- & -- & -- \\
DeepGP & \textbf{7.50} & 4.49 & 0.08 & \textbf{6.86} & 4.08 & 0.07 & \textbf{7.04} & 4.19 & 0.07 \\
BNN    & \textbf{7.42} & 4.98 & 0.08 & \textbf{7.72} & 5.31 & 0.08 & \textbf{7.57} & 5.16 & 0.07 \\
RLGP   & \textbf{7.09} & 3.94 & 0.20 & \textbf{6.66} & 3.75 & 0.20 & \textbf{6.93} & 3.92 & 0.25 \\
\midrule[1.5pt] 

 & \multicolumn{3}{c}{$d=350$}
 & \multicolumn{3}{c}{$d=400$}
 & \multicolumn{3}{c}{$d=500$} \\
\cmidrule(lr){2-4}\cmidrule(lr){5-7}\cmidrule(lr){8-10}
 & \textbf{MSE} & CRPS & Time
 & \textbf{MSE} & CRPS & Time
 & \textbf{MSE} & CRPS & Time \\
\midrule
JGP-L  & -- & -- & -- & -- & -- & -- & -- & -- & -- \\
JGP-Q  & -- & -- & -- & -- & -- & -- & -- & -- & -- \\
DeepGP & \textbf{7.10} & 4.23 & 0.08 & \textbf{7.37} & 4.38 & 0.09 & \textbf{6.69} & 3.99 & 0.14 \\
BNN    & \textbf{8.04} & 5.58 & 0.08 & \textbf{7.69} & 6.03 & 0.06 & \textbf{8.22} & 5.21 & 0.07 \\
RLGP   & \textbf{7.05} & 3.98 & 0.25 & \textbf{7.32} & 4.30 & 0.31 & \textbf{6.67} & 4.41 & 0.27 \\
\bottomrule
\end{tabular}%
}
\caption{Performance comparison of competitive methods in high dimensions ($d=10$ to $500$). }
\label{tab:hd_results_rmse_crps_time_blocks}
\end{table*}

\begin{figure}[ht!]
  \centering
  \includegraphics[width=\linewidth]{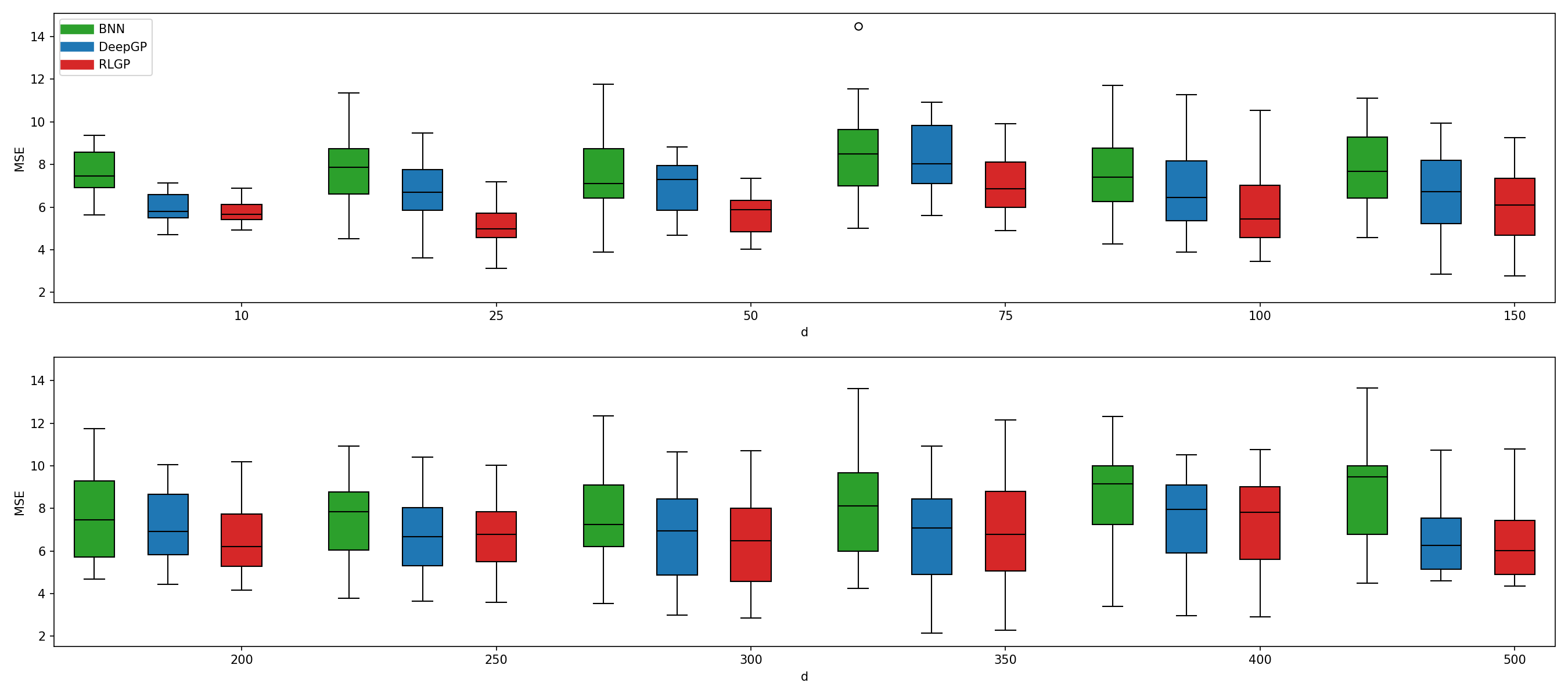}

  \caption{Box plots comparing BNN, DeepGP, and RLGP, where each box summarizes 20 replications. The top panel shows results for $d\in\{10, 25, 50, 75, 100, 150\}$, and the bottom panel shows $d\in\{200, 250, 300, 350, 400, 500\}$.}
  \label{fig:mse_box_stacked}
\end{figure}

We also evaluated the peak physical RAM usage required by each method across varying dimensions. While memory usage naturally increased with $d$ for all models, RLGP consistently demonstrated superior memory efficiency. For instance, at $d=10$, RLGP required only 0.20 GB of RAM, compared to   0.61 GB for DeepGP and 0.30 GB for BNN. This advantage became more pronounced at higher dimensions; at $d=1000$, RLGP's peak usage was 0.43 GB, substantially lower than BNN's 0.78 GB and DeepGP's 1.75 GB. These results, obtained without utilizing GPU memory, highlight RLGP's significantly lighter memory footprint, making it particularly suitable for environments with limited RAM resources.

Beyond runtime and memory efficiency, RLGP also offers  advantages in parameter selection complexity and sensitivity. Deep learning approaches, exemplified by DeepGP, often require tuning multiple structural hyperparameters like the number of layers ($L$), heavily impacting performance and resource demands. For instance, increasing DeepGP's configuration from $L=5$ layers (used for optimal results in Table \ref{tab:hd_results_rmse_crps_time_blocks}) to $L=7$ layers increased runtime by approximately 36\% at $d=400$, and this larger configuration failed to run at higher dimensions due to memory constraints, illustrating the careful parameter selection required there.  Furthermore, deep learning outcomes can be sensitive to the specific software library used; switching the DeepGP implementation from GPflux to PyDeepGP increased the MSE by over 15\%. In stark contrast, RLGP involves only a single primary hyperparameter, $q$, and demonstrated remarkable robustness. Even when bypassing its adaptive $q$-schedule and using fixed values like $q=0.15n$ or $q=0.20n$, the MSE changed by less than 4\%. 

In summary, RLGP offers a compelling overall balance across key performance metrics.  RLGP excelled in predictive accuracy and demonstrated superior memory efficiency. It also features a simpler, more robust tuning process compared to the complex and sensitive parameter selection often required for deep learning approaches. Although parallelization could further enhance its speed, RLGP's current serial implementation already confirms its practical scalability and effectiveness.

\section{Summary}

Response surfaces that contain  regime shifts, and other localized irregularities often overwhelm standard Gaussian-process emulators in industrial and engineering applications, leading to poor accuracy, limited robustness, and high computational cost. To address these challenges, this paper introduced the Robust Local Gaussian Process (RLGP), a novel framework tailored for modeling response surfaces that exhibit abrupt jumps and heterogeneity.

RLGP sets itself apart by utilizing a mean-shift robustification technique combined with a multivariate perspective transformation. It also incorporates an $\ell_{0}$-type regularization, enabling it to effectively manage nonstationary and discontinuous surfaces.
These features empower RLGP to effectively identify and compensate for anomalous observations, which are often prevalent due to imperfections in neighborhood selection and the inherent variability of data.

 At its core, RLGP features an optimization-based algorithm that achieives adaptive nearest-neighbor selection with sparsity-driven iterative quantile thresholding, ensuring guaranteed convergence. This innovative design establishes RLGP as one of the few methods capable of managing complex response curves across hundreds of dimensions, delivering superior prediction accuracy while maintaining exceptional efficiency. In contrast, many existing methods in this area are limited to handling only up to 10 dimensions and often struggle to accurately model response surfaces that exhibit sudden changes and irregularities.

RLGP is designed without field-specific assumptions, enhancing its versatility across diverse domains. It offers precise predictions coupled with reliable uncertainty quantification, meeting essential demands in environments where data are high dimensional or reducing computational costs is crucial. Based on our experience,  RLGP excels in areas such as real-time monitoring and control in digital twins,  image reconstruction, and materials science. Future efforts will  expand RLGP to encompass downstream tasks including active learning, model calibration, and sensitivity analysis  to enhance surrogate modeling techniques in managing complex, high-dimensional systems.

\section*{Acknowledgement}
\noindent The first author would like to thank Dr. Hui Wang for the financial support. Additionally, we acknowledge Dr. Park for supplying three real datasets used in this study. The authors thank the anonymous reviewers for their constructive comments, which have led to a significant improvement in the manuscript.

\bibliographystyle{cas-model2-names}

\bibliography{rlgp-refs}

\end{document}